\documentclass[11pt]{article}
\usepackage{graphicx} 

\usepackage{natbib}
\usepackage{times}
\usepackage{amsmath}
\usepackage{amssymb}
\usepackage{amsthm}
\usepackage{xcolor}
\usepackage{thm-restate}
\usepackage{mathtools}
\usepackage{thmtools}
\usepackage{enumitem}
\usepackage{circledsteps}
\usepackage{subcaption}
\usepackage{booktabs}
\usepackage{tikz}
\usetikzlibrary{arrows.meta,positioning,calc}
\usepackage{ifthen}
\newboolean{isSingleColumn}

\usepackage{bm}
\usepackage{tcolorbox}

\usepackage{algorithm}
\usepackage[noend]{algpseudocode}

\usepackage[colorlinks, linkcolor=magenta, citecolor=magenta]{hyperref}

\newcommand{\vx}{\mathbf{x}}
\newcommand{\vtheta}{\boldsymbol{\theta}}
\newcommand{\loglike}{\ell_i}        
\newcommand{\poploglike}{\bar{\ell}} 
\newcommand{\losslike}{f_i}          
\newcommand{\losspop}{F}             
\newcommand{\fulllike}{\mathcal{L}} 

\theoremstyle{plain}

\theoremstyle{definition}
\newtheorem{definition}{Definition}
\newtheorem{assumption}{Assumption}

\newtheorem{example}{Example}

\usepackage{bbm}

\newcommand{\pr}{\mathbb{P}}
\newcommand{\cL}{\mathcal{L}}
\newcommand{\truncexp}{\mathsf{TruncExp}}
\newcommand{\E}{\mathbb{E}}

\newcommand{\vbeta}{\boldsymbol{\beta}}
\newcommand{\expshort}{e^{\vx_i^\top\vtheta}}

\usepackage[nameinlink,capitalize]{cleveref}

\crefname{theorem}{Theorem}{Theorems}
\crefname{proposition}{Proposition}{Propositions}
\crefname{lemma}{Lemma}{Lemmas}
\crefname{definition}{Definition}{Definitions}
\crefname{assumption}{Assumption}{Assumptions}
\crefname{remark}{Remark}{Remarks}
\crefname{example}{Example}{Examples}

\makeatletter
\if@twocolumn
    \setboolean{isSingleColumn}{false}
\else
    \setboolean{isSingleColumn}{true}
\fi

\newcommand{\adaptiveHeader}[1]{%
    \ifthenelse{\boolean{isSingleColumn}}{%
        \paragraph{#1} 
    }{%
        \textbf{#1}    
    }%
}
\newcommand{\imgwidth}{\linewidth}
\newcommand{\tabwidth}{\linewidth}

\usepackage{fullpage}
\usepackage{authblk}

\title{Efficient Learning of Survival Analysis Models Under Outcome-Dependent Truncation and Censoring}

\author[1]{Itai Zilberstein\thanks{Equal contribution. Correspondence to \texttt{izilbers@cs.cmu.edu, ianagnos@cs.cmu.edu}.}}
\author[1]{Ioannis Anagnostides\protect\footnotemark[1]}
\author[2]{Arman Kilic}
\author[1,3]{Tuomas Sandholm}

\affil[1]{Department of Computer Science, Carnegie Mellon University, Pittsburgh, PA}
\affil[2]{Department of Surgery, Division of Cardiothoracic Surgery, Medical University of South Carolina, Charleston, SC}
\affil[3]{\small{Additional affiliations: Strategy Robot, Inc., Strategic Machine, Inc., Optimized Markets, Inc.}}

\begin{document}

\maketitle
\thispagestyle{empty}

\begin{abstract}
    Modern machine learning models are typically trained on data assumed to be representative of the predictive tasks the model will face in the future, but practical datasets violate this assumption. We focus on \textit{outcome-dependent truncation} in which an observation appears in the training data depending on its label. Our work connects two largely distinct literatures: efficient learning under truncation and survival analysis. 
    Existing methods for truncated survival analysis do not characterize the optimization landscape and instead rely on iterative procedures such as expectation maximization.
    We address this gap and prove that Cox regression can be efficiently learned under truncated and right-censored data using first-order optimization when the baseline hazard is known (despite nonconvexity).  In contrast, when the baseline hazard is learned jointly, we construct identifiable and correctly specified models whose negative log-likelihood contains strict suboptimal local minima. 
    On experiments with four survival datasets, we show that accounting for truncation reduces the prediction bias induced by naive Cox regression and extends to neural networks. We also show orders of magnitude speedups compared to standard expectation maximization.
\end{abstract}

\clearpage
\setcounter{page}{1}

\section{Introduction}

Modern machine learning models are typically optimized using large amounts of data that are assumed to be representative of the predictive tasks the model will face in the future. However, many practical datasets violate this assumption; observations are often selected non-randomly. If this selection is not accounted for during training, bias can propagate into a model's predictions~\citep{Vanderschueren23:Accounting}. When predictions inform high-stakes decisions, such as treatment decisions, organ allocation, and financial decision making, these biases can have adverse consequences.

\textit{Truncation} is a common form of selection bias in which the training data is filtered according to a known (or unknown) mechanism~\citep{Tsai87:Note,Breen96:Regression}. Truncation can take many forms, and we focus on \textit{outcome-dependent truncation}, in which the data is filtered based on the value of the label. Many different reasons can cause truncation to occur, including constrained data collection, study design, or legal restrictions. For example, a medical study starting today can only enroll patients who are still alive; patients who died beforehand are never observed. More generally, datasets collected based on a selection criterion can overrepresent individuals whose outcomes make them more likely to be selected and observed in training data.  Unlike censoring, where an observation remains in the dataset but its outcome is only partially observed (\textit{e.g.}, due to loss of follow-up), truncation removes the observation entirely. Naively fitting a predictive model to the observed samples can yield a model that performs well on the selected samples, but is biased for the underlying population. 

Survival analysis is an especially important setting in which truncation and censoring arise together. A survival model predicts a distribution over time until an event occurs for a patient, such as death, disease progression, or graft failure. 
Survival analysis is particularly relevant in transplantation, where post-transplant outcomes are observed only for patients who are selected to receive an organ, and selection depends in part on characteristics predictive of transplant benefit and survival~\citep{OPTN25:ContinuousDistribution,Gottlieb17:Lung,Kamath01:Model,Zilberstein26:Learning}. Patients predicted to have poor post-transplant outcomes may be less likely to receive a transplant. Conversely, patients predicted to live the longest following a transplant may be healthy enough to remain on the waitlist while more urgent candidates are prioritized, making them less likely to be transplanted.~\Cref{ex:bias} illustrates how a selection mechanism can bias a predictive model. 

\begin{example}[Bias from left truncation]
\label{ex:bias}
Consider a hypothetical study of graft survival following heart transplantation. Suppose a large historical registry contains transplant outcomes that can be used to estimate the baseline graft survival distribution. A new follow-up study begins three years later to collect additional donor and recipient covariates from patients in this registry. The data for this new study is left truncated: only patients who survive until enrollment can enter the study, while patients who die within the first three years are entirely absent.

We fit a Cox regression model to predict graft survival based on two newly collected covariates.  Because graft survival itself determines whether a patient appears in the sample, fitting a standard Cox model produces biased coefficients, as shown in~\Cref{fig:example}. Accounting for the truncation in the likelihood instead recovers the true coefficients. This bias can have important downstream consequences: in our example, the naively trained model substantially overestimates graft survival for some patients. This occurs even without censoring and under simple one-sided left truncation. 
\end{example}

\ifthenelse{\boolean{isSingleColumn}}{
    \renewcommand{\imgwidth}{0.4\linewidth}
}{
    \renewcommand{\imgwidth}{0.48\linewidth}
}

\begin{figure}[t]
    \centering
    \includegraphics[width=\imgwidth]{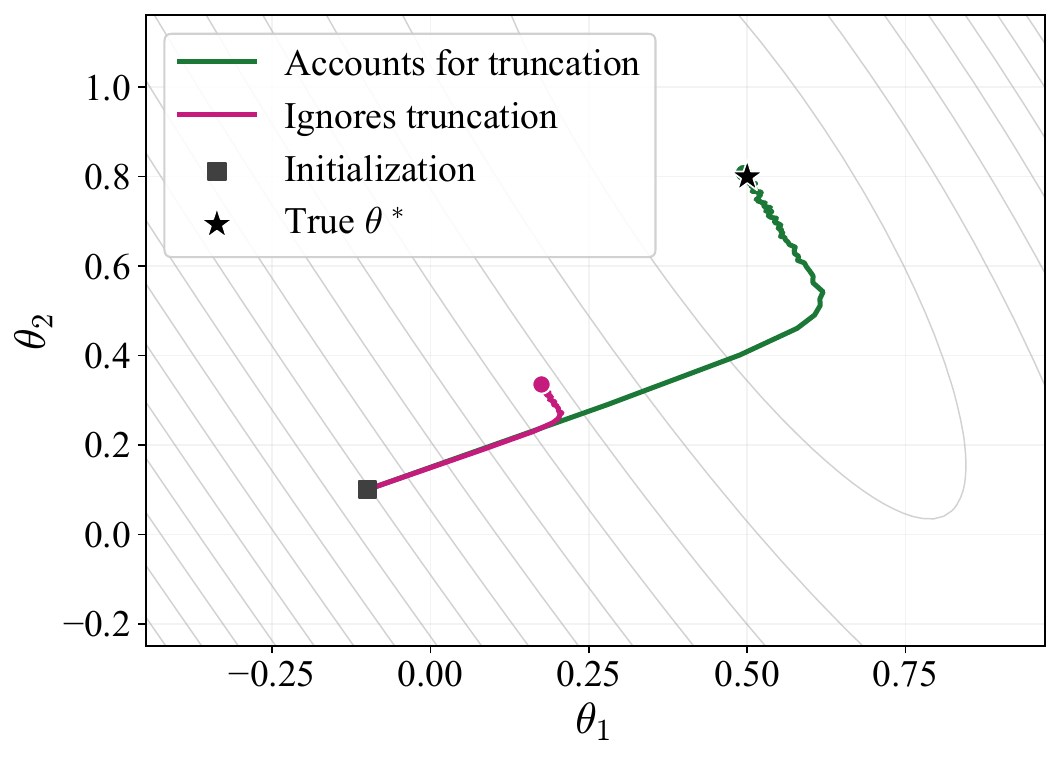}
    \includegraphics[width=\imgwidth]{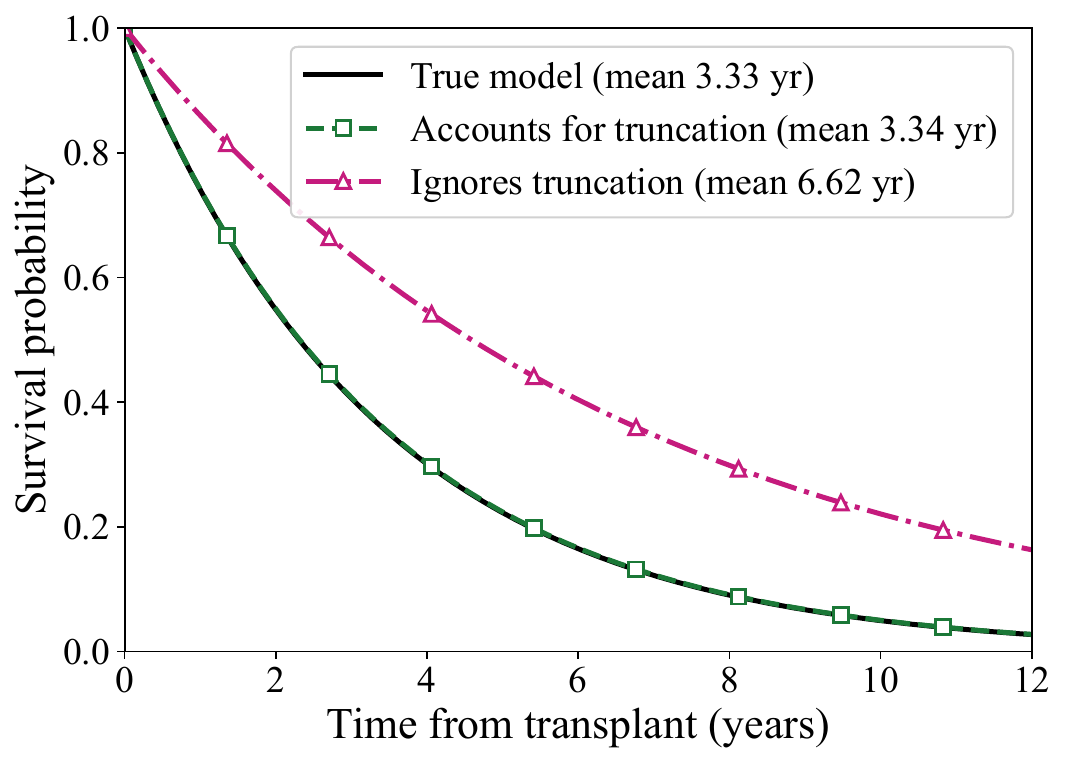}
    \caption{Illustrative example of learning a Cox model with truncated samples. When the likelihood accounts for truncation, SGD converges to the true regression coefficients, whereas the standard Cox likelihood converges to biased parameters (left). The resulting parameter bias can translate into substantial errors in predicted graft survival for individual patients (right).}
    \label{fig:example}
\end{figure}

There is a long literature in statistical estimation of survival models from truncated data, including left, right, and double truncation, biased sampling, and censoring~\citep{Rennert18:Cox,Rennert22:Cox,Shen19:Pseudo}. These methods provide consistent estimators, but largely focus on statistical identifiability and asymptotic properties rather than the global optimization problem. In practice, learning such a model can be computationally challenging, relying on methods such as \textit{expectation maximization (EM)}~\citep{Rennert18:Cox,Rennert22:Cox}, which can become costly in high-dimensional settings. In parallel to this body of work, machine learning researchers have developed computationally efficient methods for learning regression models and exponential families from truncated samples~\citep{Daskalakis19:Computationally,Ilyas20:Theoretical,Lee23:Learning}. As datasets increasingly grow in size and dimension, efficient methods for learning survival models under truncation is an important intersection of these two literatures. We provide a broader discussion of these two literatures in~\Cref{sec:related}.

\paragraph{Our contributions}
We study whether Cox regression can be learned efficiently from outcome-dependent truncated and right-censored data. We consider a general truncation mechanism in which an observation is retained probabilistically if it falls in one of multiple intervals. This captures standard left, right, and interval truncation as special cases, while also permitting disconnected multiple intervals, which may arise from merging multiple data collection processes.


Our main positive result shows that, despite nonconvexity, the population negative log-likelihood has a unique stationary point at the true parameter and its gradient always has a component pointing away from it when the baseline hazard is known (\Cref{thm:cox-population-uniqueness}). This implies convergence of projected \textit{stochastic gradient descent (SGD)} to the true parameter under standard conditions (\Cref{cor:cox-sgd-finite-time}). The analysis relies on a connection between Cox regression and weighted truncated exponential distributions that remains valid under right censoring.

We then show that these guarantees need not extend to an unknown baseline. In the worst case, the model can become non-identifiable (\Cref{ex:baseline-nonidentifiability}). We further show that even when the model is correctly specified, identifiable, and there is no censoring, jointly estimating the regression coefficient and baseline can create a suboptimal local minimum (\Cref{thm:finite-spurious-cox}). We also show that after optimizing over a class of baseline hazards for each regression coefficient, the resulting profiled population loss can be flat and suboptimal over an entire interval of incorrect parameters (\Cref{thm:flat-unknown-baseline-profile}). 

This characterization highlights when, and when not, Cox regression is efficiently learnable under truncation. A known or accurately estimated baseline hazard can arise naturally from external data. For example, a large disease registry may provide reliable population-level survival information while a smaller study collects richer biomarkers or clinical covariates~\citep{Zhou06:Cox,Huang16:Efficient,Royston13:External}. In such settings, the registry can supply the baseline survival distribution while the smaller, potentially truncated cohort is used to estimate covariate effects.

Finally, we illustrate that our theoretical guarantees translate into practical gains on semi-synthetic experiments with four clinical datasets. Optimizing the truncated Cox likelihood consistently reduces prediction error compared to optimizing the standard likelihood as truncation levels increase. We also show that our results extend practically to neural network survival models, which can be trained efficiently with gradient-based methods under truncation. Compared with a commonly used EM approach, our first-order methods are substantially faster. On an enlarged dataset with hundreds of thousands of observations and thousands of covariates, they reduce training time by roughly $100\times$ while maintaining or improving predictive accuracy.

\section{Preliminaries}
\label{sec:prelims}

We begin by recalling the basic Cox proportional hazards model without truncation. Let $n$ be the number of independent samples. For each sample $i=1, \dots, n$, let $T_i^*$ denote the true event time, $C_i\in[0,\infty]$ the censoring time, and $\vx_i\in\mathbb R^d$ the covariate vector. We observe $(\vx_i,T_i,\delta_i)$, where $T_i=\min(T_i^*,C_i)$ and $\delta_i=\mathbbm{1}\{T_i^*\le C_i\}$. Let $\vtheta\in\mathbb{R}^d$ be the underlying regression parameter. All likelihoods are conditional on the observed covariates.

We assume a known baseline hazard, $h_0(t) > 0$. The hazard for sample $i$ is given by
$ h_i(t)=h_0(t)e^{\vx_i^{\top}\vtheta}$. 
We denote the cumulative baseline hazard by
$ H_0(t)=\int_0^t h_0(u)du$. 
We assume $h_0$ is locally integrable, $H_0(t)<\infty$ for every finite $t$, and $H_0(t)\to\infty$ as $t \to\infty$. The survival function for sample $i$ is given by 
$S_i(t)=\exp\Big(-\int_0^t h_i(u)du\Big)=\exp\big(-e^{\vx_i^{\top}\vtheta} H_0(t)\big)$. 
A standard assumption in survival analysis is that the censoring time is conditionally independent of the event time.
\begin{assumption}[Conditionally independent censoring]
\label{assum:censoring}
Conditional on $\vx_i$, the censoring time $C_i$ is independent of $T_i^*$, and its conditional distribution does not depend on $\vtheta$. 
\end{assumption}

We now consider a more complex setting involving both right censoring and \emph{interval truncation}.

\begin{definition}[Interval truncation with right censoring]
Let $T_i^*$ denote the true event time for sample $i$. Consider $K$ pairwise disjoint intervals $(t_{0,k},t_{1,k})$, where $0 \le t_{0,k} < t_{1,k} \le \infty$, and known probabilities $q_k \in [0,1]$, not all zero. Conditional on $T_i^*=t$, sample $i$ is included with probability
\[ \rho(t) = \sum_{k=1}^K q_k \mathbbm{1}\{t \in (t_{0,k},t_{1,k})\}. \]
The inclusion indicator is conditionally independent of $C_i$ given $(T_i^*,\vx_i)$.
\end{definition}

For a fixed individual $i$, we condition on $\vx_i$ throughout. Distributions and expectations for sampled individuals are additionally conditional on sample inclusion, and we suppress this conditioning in the notation. The hazard and survival functions, and the probabilities defining $\pi_i(t\mid\vtheta)$, refer to the population before conditioning on inclusion.

Define
$ \pi_i(t \mid \vtheta) = \sum_{k=1}^K q_k \pr(T_i^*>t, t_{0,k}<T_i^*<t_{1,k}) $
as the probability that sample $i$ is included and its event time exceeds $t$. We can equivalently write
$ \pi_i(t \mid \vtheta) = \sum_{k \mid t\le t_{0,k}} q_k\bigl[S_i(t_{0,k})-S_i(t_{1,k})\bigr]  +\sum_{k \mid t_{0,k} < t < t_{1,k}} q_k\bigl[S_i(t)-S_i(t_{1,k})\bigr]$. 
The first summation includes intervals after $t$, and the second includes the interval containing $t$.
The probability of inclusion is
\[ \pi_i(\vtheta) = \pi_i(0 \mid \vtheta) = \sum_{k=1}^K q_k \bigl[S_i(t_{0,k})-S_i(t_{1,k})\bigr]. \]

The likelihood under interval truncation, up to factors independent of
$\vtheta$, is 
\[ \fulllike(\vtheta) \propto \prod_{i=1}^n \frac{ [h_i(T_i)S_i(T_i)]^{\delta_i} [\pi_i(T_i \mid \vtheta)]^{1 -\delta_i}}{\pi_i(\vtheta)}. \] 

Let $\loglike(\vtheta)$ denote the log-likelihood contribution of observation $i$, and define its negative log-likelihood by $\losslike(\vtheta) := -\loglike(\vtheta)$. The population log-likelihood and population negative log-likelihood are respectively
\[ \poploglike(\vtheta) := \E \bigl[\loglike(\vtheta)\bigr],
\text{ and } 
\losspop(\vtheta) := \E \bigl[\losslike(\vtheta)\bigr] = -\poploglike(\vtheta). \]

We provide the derivation of both the standard Cox likelihood and the likelihood under interval truncation in~\Cref{sec:appendix-prelims}.

For the population results below, we assume that $\pr_{\vtheta^*}(\delta=1\mid\vx)>0$ for almost every covariate value $\vx$ in the sampled population. This condition will be needed to identify the true population parameter and rules out the case where every sample is censored for a particular covariate.

We also make the following two standard assumptions for our population results.

\begin{assumption}[Non-degeneracy]
\label{assum:covariate-nondegeneracy}
Under the sampled population, $\E[ \|\vx \|_2^2 ]<\infty$ and $\E[\vx\vx^\top]\succ 0$.
\end{assumption}

\begin{assumption}[Regularity]
\label{assum:regularity}
The population loss is continuously differentiable, its gradient may be computed by differentiating under expectation, $\nabla \losspop(\vtheta) = \E[\nabla \losslike(\vtheta)]$, and all expectations appearing in the analysis are finite.
\end{assumption}

\section{Main results}

In this section, we give the main result and proof overview.  Unless stated otherwise, all probabilities and expectations are taken with respect to the distribution with the true parameters and conditioned on $\vx_i$ when covariates are fixed. The hazard, survival, and inclusion probabilities are defined with respect to the population before conditioning on sample inclusion.

A natural route to efficient optimization would be to show that the truncated population log-likelihood remains concave. However, under interval truncation, this is false even in the simplest one-dimensional setting without censoring. 

\begin{example}[Loss of concavity under interval truncation]
\label{ex:concavity}
Consider a one-dimensional Cox model with covariate $x=1$, baseline hazard
$h_0(t)=1$, no censoring, and truncation to $(0,1)$. The sampled
log-likelihood of an event at time $T$ is
\[ \ell(\theta;T) = \theta-Te^\theta-\log\left(1-e^{-e^\theta}\right). \]
Let the true parameter be $\theta^*=\log 3$. Conditional on inclusion, $T$ is exponentially distributed with rate $3$, truncated to $(0,1)$, and therefore
$\E[T] = \frac{1}{3}-\frac{1}{e^3-1}$.
The population log-likelihood is
\[ \poploglike(\theta) = \E[\ell(\theta;T)] = \theta- \E[T] e^\theta-\log\left(1-e^{ -e^\theta}\right). \]
Its second derivative at $\theta=0$ is
\[ \poploglike''(0) = -\E[T] +\frac{1}{(e-1)^2} \approx -0.281+0.339 >0. \]
Therefore, the population log-likelihood is locally convex at $\theta=0$ and therefore is not globally concave.
\end{example}

\ifthenelse{\boolean{isSingleColumn}}{
    \renewcommand{\imgwidth}{0.4\linewidth}
}{
    \renewcommand{\imgwidth}{\linewidth}
}

\begin{figure}[h!]
    \centering
    \includegraphics[width=\imgwidth]{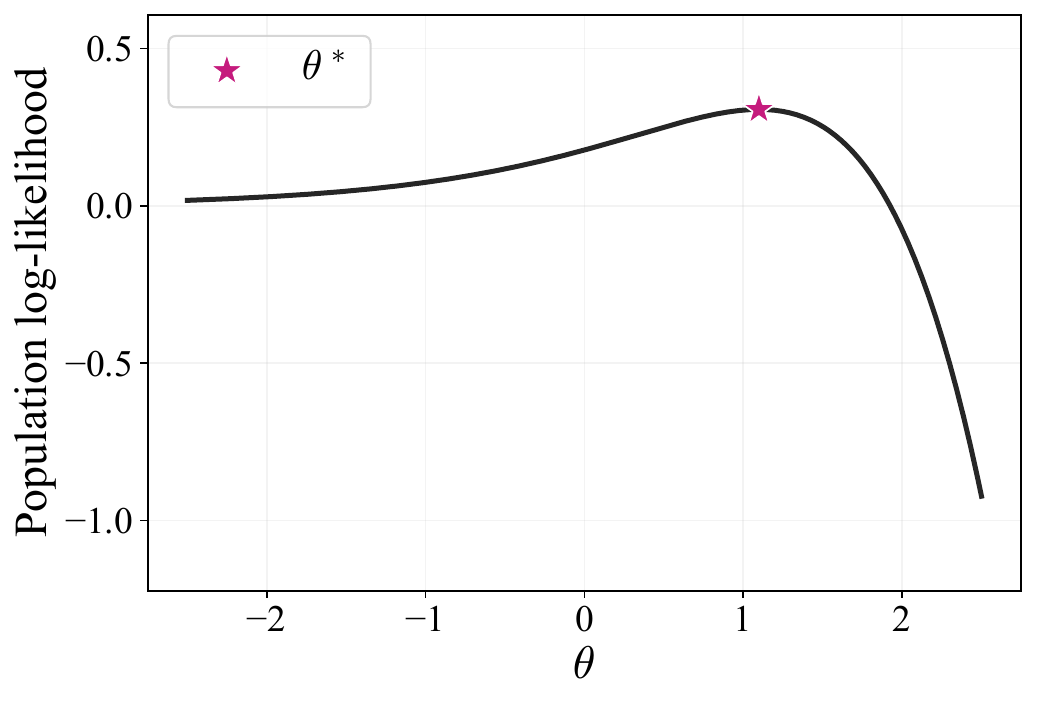}
    \caption{Population log-likelihood from~\Cref{ex:concavity}.}
    \label{fig:example-likelihood}
\end{figure}

From this example, it is clear to see that our result cannot follow from concavity alone, and we visually show the population likelihood from~\Cref{ex:concavity} in~\Cref{fig:example-likelihood}. Instead, we show that the likelihood has a nice property: at every incorrect parameter value, the gradient points in the direction of the true parameter value. This property is sufficient to rule out spurious stationary points and to guarantee convergence of projected SGD.

\paragraph{Proof overview} Our analysis proceeds in three main steps. First, we establish several properties of weighted truncated exponential distributions that will ultimately be applied to our truncated Cox model. \Cref{lem:truncated-exponential-mean} characterizes their mean, \Cref{lem:truncated-exponential-monotonicity} shows that increasing the exponential rate decreases the expectation of any nondecreasing function, and \Cref{lem:truncated-exponential-censoring} shows how right censoring can be represented by a nondecreasing transformation that preserves the mean. Second, we connect these results to the Cox model by showing that after transforming the event time by the cumulative baseline hazard $H_0$, the resulting distribution is a weighted truncated exponential distribution (\Cref{lem:cox-truncated-exponential}).~\Cref{lem:cox-likelihood-gradient} then expresses the gradient of the negative log-likelihood for a single observation in terms of the corresponding means. Combining these ingredients, \Cref{thm:cox-conditional-population} proves that, for fixed covariates, the population gradient always has the sign that causes gradient
descent to move toward the true linear predictor. Finally, \Cref{thm:cox-population-uniqueness} leverages this property to show that the true regression parameter is the unique stationary point of the multivariate population loss and~\Cref{cor:cox-sgd-finite-time} uses this property to prove that projected SGD converges to the true optimum. 

\subsection{Weighted truncated exponentials}
\label{sec:weighted-trunc-exp}

We first establish useful properties of \textit{weighted truncated exponential distributions}, which we will ultimately apply to the Cox model. These properties describe the mean of the sampled distribution, how expectations change with the exponential rate, and how to account for right censoring.

Fix $K$ pairwise disjoint intervals $[A_k,B_k]$, where $0\le A_k<B_k\le\infty$, and fixed probabilities $q_k\in[0,1]$, not all zero. The weighted truncated exponential distribution with rate $\lambda>0$ conditional on retention, denoted $\truncexp\left(\lambda,\{([A_k,B_k],q_k)\}_{k=1}^K\right)$, is obtained by retaining an exponential random variable with probability $q_k$ when it lies in the interval $(A_k,B_k)$, and with probability zero otherwise. When every $q_k=1$, this is simply an exponential distribution conditioned to lie in the union of the intervals. When $B_k=\infty$, the expressions $e^{-\lambda B_k}$ and $B_k e^{-\lambda B_k}$ are equal to zero.

The mean of an ordinary exponential distribution is $1/\lambda$. Our first lemma expresses the mean of the weighted truncated exponential distribution (after selection) as this quantity plus a correction determined by the intervals $[A_k,B_k]$ and probabilities $q_k$. 

\begin{restatable}{lemma}{TruncExpMean}
\label{lem:truncated-exponential-mean}
Let $U \sim \truncexp\left(\lambda,\{([A_k,B_k],q_k)\}_{k=1}^K\right)$ be a random variable following a weighted truncated exponential distribution with rate $\lambda>0$. Its expected value, denoted by $\mu(\lambda)$, satisfies
\[ \mu(\lambda) = \frac{\sum_{k=1}^K q_k \left(A_k e^{-\lambda A_k}-B_k e^{-\lambda B_k}\right)}{\sum_{k=1}^K q_k \left(e^{-\lambda A_k}-e^{-\lambda B_k}\right)} + \frac{1}{\lambda}. \]
\end{restatable}

The formula for the mean allows us to evaluate expectations, but understanding the direction of the likelihood gradient
requires comparing the expectation across different rates. One observation is that if we fix the intervals and weights, increasing $\lambda$ puts more weight on smaller values. The next lemma shows that this ordering persists despite the gaps between intervals and their different sampling probabilities: the expectation of any fixed nondecreasing function cannot increase with the rate. The comparison is strict whenever the function is not constant. This strictness will be important for ruling out stationary points.

\begin{restatable}{lemma}{TruncExpMonotonicity}
\label{lem:truncated-exponential-monotonicity}
Let $0<\lambda_1<\lambda_2$ and $U_j\sim \truncexp \left(\lambda_j,\{([A_k,B_k],q_k)\}_{k=1}^K \right)$ for $j = 1,2$. Let $\varphi:[0,\infty)\to\mathbb R$ be a nondecreasing function such that $\E[|\varphi(U_j)|]<\infty$ for $j=1,2$. Then
\[ \E[\varphi(U_1)]\ge\E[\varphi(U_2)]. \]
The inequality is strict whenever $\varphi(U_1)$ is not almost surely constant.
\end{restatable}

We next account for right censoring. For a fixed censoring threshold $c$, a value $u\le c$ is observed exactly, whereas
a value $u>c$ is known only to exceed $c$. From this information, we cannot directly use the value of $U$ in every observation.

For a given rate $\lambda$, we represent this partial information by retaining each exact value and replacing
a censored value with the conditional tail mean $\mu_c(\lambda)=\E[U\mid U>c]$. Above the censoring threshold, the resulting function $m_{\lambda,c}$ depends only on $c$ and the rate, not on the unobserved value of $U$. The next lemma establishes two properties of this representation. First, it is nondecreasing in the underlying value, so we can eventually apply~\Cref{lem:truncated-exponential-monotonicity}. Second, its expectation equals $\mu(\lambda)$. If exact observations have positive probability, the function is also nonconstant, providing the strictness needed later.

\begin{restatable}{lemma}{TruncExpCensoring}
\label{lem:truncated-exponential-censoring}
Let $U\sim\truncexp\left(\lambda,\{([A_k,B_k],q_k)\}_{k=1}^K \right)$ with expected value $\mu(\lambda)$. For a fixed censoring threshold $c \ge 0$ with $\pr(U>c) > 0$, define
\[ \mu_c(\lambda):= \E[U \mid U>c]  \text{ and } m_{\lambda,c}(u) =
\begin{cases}
u, & u\le c,\\
\mu_c(\lambda), & u>c.
\end{cases} \]
Then $m_{\lambda,c}$ is nondecreasing and satisfies
\[ \E[m_{\lambda,c}(U)]=\mu(\lambda). \]
If $\pr(U\le c)>0$, then $m_{\lambda,c}$ is nonconstant on the set $\mathcal S=\bigcup_{k \mid q_k>0}(A_k,B_k)$.
\end{restatable}

If $\pr(U>c)=0$, every sampled value is uncensored. In this case, the same conclusions hold by defining $m_{\lambda,c}(u)=u$. This convention also covers the case where $c=\infty$.

The expectation identity extends directly to a random censoring threshold $C$ independent of $U$. Conditional on any fixed threshold, the distribution of $U$ is unchanged, so \Cref{lem:truncated-exponential-censoring} gives
\[ \E[m_{\lambda,C}(U)] = \E\left[\E[m_{\lambda,C}(U)\mid C]\right] = \mu(\lambda). \]
Averaging over independent censoring preserves the mean identity.

The lemmas in this section provide the comparison needed in the population likelihood analysis. We will construct $m_{\lambda,c}$ using a rate $\lambda$ and compare its expectation under the candidate and true distributions. Under the candidate distribution, its expectation is exactly $\mu(\lambda)$ and under the true distribution,~\Cref{lem:truncated-exponential-monotonicity} determines whether the expectation is larger or smaller. The function $m_{\lambda,c}$ is held fixed. Only the distribution used to take its expectation changes. We next connect this argument to the Cox model by transforming it into a weighted truncated exponential.

\subsection{Connecting to Cox regression}

We now connect the weighted truncated exponential distributions from the previous subsection to the Cox model. The key observation is that transforming an event time by the cumulative baseline hazard removes the baseline hazard from the density. Without truncation, this is the familiar fact that $H_0(T_i^*)$ is exponentially distributed with rate $e^{\vx_i^\top\vtheta}$. Under interval-weighted truncation, the same transformation maps each sampling interval $(t_{0,k},t_{1,k})$ to $(H_0(t_{0,k}),H_0(t_{1,k}))$ while preserving its inclusion probability $q_k$. The resulting distribution is a weighted truncated exponential.

\begin{restatable}{lemma}{CoxTruncExp}   
\label{lem:cox-truncated-exponential}
Fix the covariates $\vx_i$ and suppose $T_i^*$ follows the Cox model with parameter $\vtheta$ and known baseline hazard $h_0(t)$. Then, conditional on sample inclusion, $H_0(T_i^*)$ is distributed as
\[ \truncexp\left(\expshort, \left\{ \left([ H_0(t_{0,k}), H_0(t_{1,k})],q_k \right)\right\}_{k=1}^K \right). \]
\end{restatable}

An important consequence of~\Cref{lem:cox-truncated-exponential} is that the transformed intervals and retention probabilities do not depend on $\vtheta$. This means that if we fix $\vx_i$, changing $\vtheta$ only changes the exponential rate $\expshort$. This allows us to compare the true and candidate event-time distributions using~\Cref{lem:truncated-exponential-monotonicity}. This is the main link between the results in~\Cref{sec:weighted-trunc-exp} and the population likelihood analysis to follow.

\subsection{Population likelihood landscape}

We next derive the gradient of the truncated Cox likelihood. The gradient takes a simple form after transforming event times by the cumulative baseline hazard. For an uncensored observation, the contribution depends on the transformed event time $H_0(T_i)$. For a censored observation, the exact event time is unknown, so the contribution instead depends on the conditional mean given that the event occurs after the censoring time. We define the mean transformed event time among included individuals by $\mu_i(\vtheta):=\E [H_0(T_i^*)\mid\vx_i]$, and the corresponding conditional tail mean by $\mu_i(t\mid\vtheta):=\E\left[H_0(T_i^*)\mid\vx_i,T_i^*>t\right]$.

\begin{restatable}[Single-observation likelihood gradient]{lemma}{CoxLikelihoodGrad}
\label{lem:cox-likelihood-gradient}
Let $\loglike(\vtheta)$ denote the log-likelihood contribution of observation $(T_i,\delta_i)$, and let $\losslike(\vtheta)=-\loglike(\vtheta)$. Then the gradient $\nabla_{\vtheta}\losslike(\vtheta)$ can be expressed as 
\[
e^{\vx_i^\top\vtheta} \left[\delta_iH_0(T_i) + (1-\delta_i)\mu_i(T_i\mid\vtheta) - \mu_i(\vtheta) \right]\vx_i. 
\]
\end{restatable}

In both the observed and censored case, the mean $\mu_i(\vtheta)$ appears. This will allow us to use the monotonicity properties from~\Cref{lem:truncated-exponential-monotonicity} to determine the direction of the population gradient.

Finally, we leverage~\Cref{lem:cox-likelihood-gradient} to prove that, for a fixed covariate vector, the population negative log-likelihood has a single unique stationary point that is its minimizer.

\begin{restatable}[Unique stationary point of the population landscape in one dimension]{theorem}{CoxStatPoint}
\label{thm:cox-conditional-population}
Fix a covariate vector $\vx$ and let $(T, \delta)$ be the observed time and event indicator. Let $z^*=\vx^\top\vtheta^*$ be the true data-generating parameter. The one-dimensional population negative log-likelihood
\[ \losspop_{\vx}(z)=\E_{\vtheta^*}[-\ell(z;T,\delta) \mid \vx ], \]
has $z^*$ as its unique global minimizer and unique stationary point. Moreover,
\[ (z-z^*)\losspop_{\vx}'(z)>0 \text{ for every } z\ne z^*. \]
\end{restatable}

The next result simply lifts this one-dimensional result to the full distribution of training samples. 

\begin{restatable}[Unique stationary point of the population loss]{theorem}{CoxPopUnique}
\label{thm:cox-population-uniqueness}
Let $\vtheta^*$ be the true data-generating parameter. Then the population negative log-likelihood $\losspop(\vtheta)$ satisfies
\[ \left\langle\nabla_{\vtheta}\losspop(\vtheta),\vtheta-\vtheta^*\right\rangle>0 \text{ for every }\vtheta\ne\vtheta^*. \]
Moreover, $\nabla_{\vtheta}\losspop(\vtheta^*)=\mathbf0$. Consequently, $\vtheta^*$ is the unique stationary point of $\losspop$. 
\end{restatable}

\subsection{Convergence of SGD}

\Cref{thm:cox-population-uniqueness} is sufficient to prove convergence of minibatch projected SGD (\Cref{alg:minibatch-sgd}). 

\begin{algorithm}[h]
\caption{Minibatch projected SGD for truncated Cox regression}
\label{alg:minibatch-sgd}
\begin{algorithmic}[1]
\Require Projection set $\Theta$, $\vtheta^{(0)}\in\Theta$,  batch size $b$, and step sizes $\eta_i$
\For{$i=0,\ldots,N-1$}
    \State Draw independent $\{(\vx_{i,j},T_{i,j},\delta_{i,j})\}_{j=1}^b$
    \For{$j=1,\ldots,b$}
        \State Compute $\lambda_{i,j}\gets e^{\vx_{i,j}^\top\vtheta^{(i)}}$ and $\mu_{i,j}(\vtheta^{(i)})$
        \If{$\delta_{i,j}=1$}
            \State $u_{i,j}\gets H_0(T_{i,j})$
        \Else
            \State $u_{i,j}\gets\mu_{i,j}(T_{i,j}\mid\vtheta^{(i)})$
        \EndIf
    \EndFor
    \State $\Delta \vtheta^{(i)} = \frac{1}{b}\sum_{j=1}^b\lambda_{i,j}\left[u_{i,j}-\mu_{i,j}(\vtheta^{(i)})\right]\vx_{i,j}$
    \State $\vtheta^{(i+1)}\gets\Pi_\Theta\left(\vtheta^{(i)}- \eta_i \Delta \vtheta^{(i)}\right)$
\EndFor
\State \Return $\vtheta^{(N)}$
\end{algorithmic}
\end{algorithm}

\begin{restatable}[Finite-time convergence of minibatch projected SGD]{corollary}{CoxSGDFiniteTime}
\label{cor:cox-sgd-finite-time}
Let $\Theta\subset\mathbb R^d$ be compact and convex with $\vtheta^*\in\Theta$, and let $\Pi_\Theta$ denote Euclidean projection onto $\Theta$. Suppose that $\nabla\losspop$ is differentiable at $\vtheta^*$, with $\nabla^2\losspop(\vtheta^*)=\E_{\vtheta^*}[\nabla^2\losslike(\vtheta^*)]$ finite. Assume also that, for some finite constant $B$,
$\E \left[\|\nabla_{\vtheta}\losslike(\vtheta)\|_2^2\right]\le B^2$ for all $\vtheta \in \Theta$. There exist constants $0< \kappa \le L<\infty$, depending only on the population model and $\Theta$, for which the following guarantee holds. Let $\vtheta^{(t)}$ denote the iterates of \Cref{alg:minibatch-sgd} with minibatch size $b$, a fixed initialization $\vtheta^{(0)}\in\Theta$, and step sizes $\eta_i=2/[\kappa(i+t_0)]$, where $t_0\ge\max\{3,2L^2/\kappa^2\}$. Then, for every $t=1,\ldots,N$,
\[ \E\left[\|\vtheta^{(t)}-\vtheta^*\|_2^2\right]\le\frac{t_0^2\|\vtheta^{(0)}-\vtheta^*\|_2^2}{t^2}+\frac{4B^2}{\kappa^2bt}. \]
\end{restatable}

\section{Impossibility results for unknown baselines}

In this section, we show that the population optimization guarantees established previously do not always hold when the baseline is unknown. Beyond optimization guarantees, an unknown baseline can introduce an even more drastic difficulty: non-identifiability~\citep{Finkelstein93:Proportional}.

\begin{example}[Non-identifiability of the baseline under truncation]
\label{ex:baseline-nonidentifiability}
Consider a truncated Cox model with $\theta=0$, no censoring, and truncation to the interval $(0,1)$. For any $a\in(0,1)$, define the baseline hazard on $[0,1]$ by $h_{0,a}(t):=\frac{a}{1-at}$. The survival function and event-time density on $[0,1]$ are $S_a(t)=1-at$ and $p_a(t)=h_{0,a}(t)S_a(t)=a$. We obtain $\Pr(0<T^*<1)=a$, and the sampled density is $p(t\mid 0<T^*<1)=\frac{p_a(t)}{\Pr(0<T^*<1)}=1$ for every $0<t<1$. Therefore, every $a\in(0,1)$ produces exactly the same sampled distribution, namely the uniform distribution on $(0,1)$, even though the baseline hazards $h_{0,a}$ are different. Each $h_{0,a}$ can be extended smoothly and strictly positively beyond $t=1$ so that its cumulative hazard diverges at infinity, without changing the sampled distribution. We conclude that the baseline hazard, and the full model, are not identifiable from the truncated data.
\end{example}

However, even when the model is identifiable, we show that other issues arise from an unknown baseline.  When the baseline hazard is known, changing the regression coefficients has a monotone effect on the event-time distribution. However, when the baseline is also estimated, an incorrect regression coefficient can be offset by a change in the baseline, and this monotone structure can disappear. We give two counterexamples illustrating different consequences of this interaction. Both have no censoring and an identifiable regression coefficient. First, we show that even when the baseline is restricted to a smooth one-parameter family, the joint population loss can have a suboptimal local minimum. 

\begin{restatable}[A suboptimal local minimum]
{theorem}{CoxUnknownBaselineLocalMin}
\label{thm:finite-spurious-cox}
There exists a correctly specified truncated Cox model with a smooth, strictly positive baseline family and an identifiable joint parameter
$\vbeta=(\theta,\gamma)$ where $\gamma$ is the parameter of the baseline family such that, under truncation to a single interval and without censoring, $\losspop$ has the true parameter $\vbeta^*$ as its unique global minimizer, but also has a point $\vbeta'$ in the interior of the parameter space with $\theta'\ne\theta^*$ satisfying
\[ \nabla \losspop(\vbeta')=\mathbf{0} \text{ and } \nabla^2 \losspop(\vbeta')\succ 0, \text{ but } \losspop(\vbeta')>\losspop(\vbeta^*). \]
\end{restatable}

In the second example,  we allow the baseline to be the best-fitting function from a broad class for each candidate regression coefficient. We show that the resulting profiled population loss can be constant and strictly suboptimal over an entire interval of incorrect regression coefficients. Neither joint estimation nor profiling over an unknown baseline preserves the favorable population geometry from the known baseline setting. 

\begin{restatable}[A suboptimal flat region after optimizing the baseline]{theorem}{CoxFlatProfile}
\label{thm:flat-unknown-baseline-profile}
Let $\mathcal H$ be the class of smooth, strictly positive baseline hazards whose cumulative hazards diverge at infinity. There exists a correctly specified truncated Cox model whose true regression coefficient $\theta^*>0$ is identifiable, but whose profiled population negative log-likelihood
\[ \widehat{\losspop}(\theta):=\inf_{h_0\in\mathcal H} \losspop(\theta,h_0) \]
is constant on $(-\infty,0]$ and satisfies $\widehat{\losspop}(\theta)>\widehat{\losspop}(\theta^*)$ for every $\theta\le0$. Consequently, every $\theta<0$ is a suboptimal stationary point of $\widehat{\losspop}$ and for every $\theta<0$,  the infimum defining $\widehat{\losspop}(\theta)$ is not attained by any baseline in $\mathcal H$.
\end{restatable}

\section{Experiments}

\begin{figure*}[b!]
    \centering
    
    \begin{subfigure}[b]{0.24\textwidth}
        \centering
        \includegraphics[width=\linewidth]{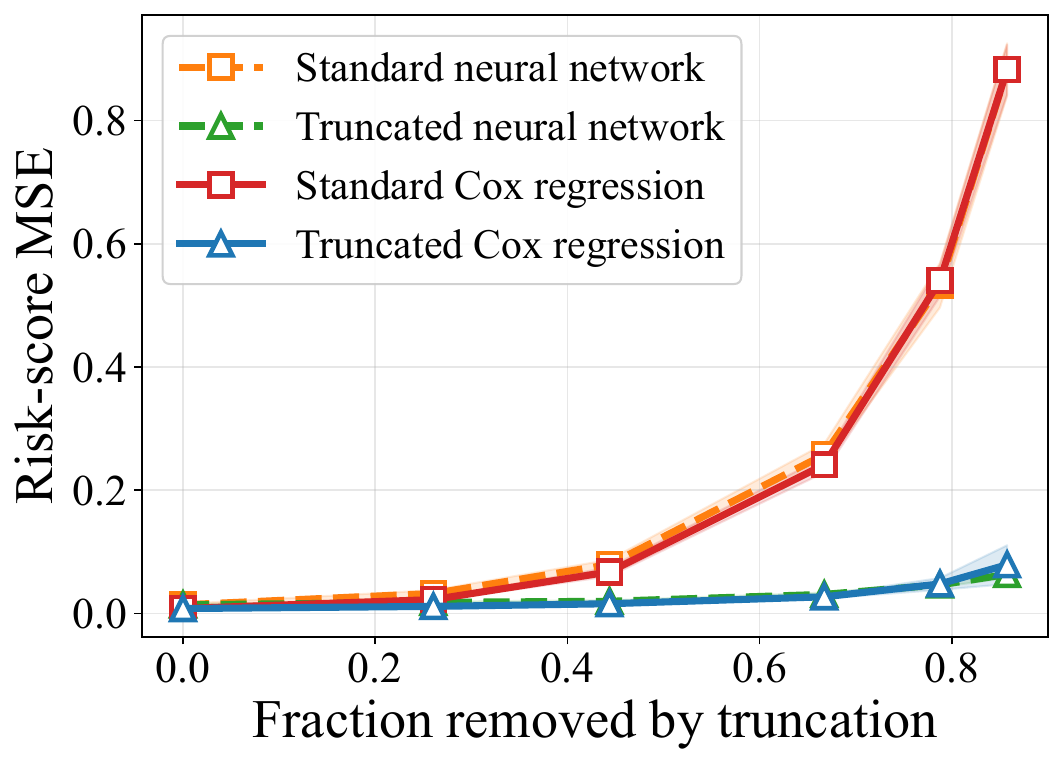}
        \caption{HEARTS.}
        \label{fig:unos}
    \end{subfigure}
    \hfill
    \begin{subfigure}[b]{0.24\textwidth}
        \centering
        \includegraphics[width=\linewidth]{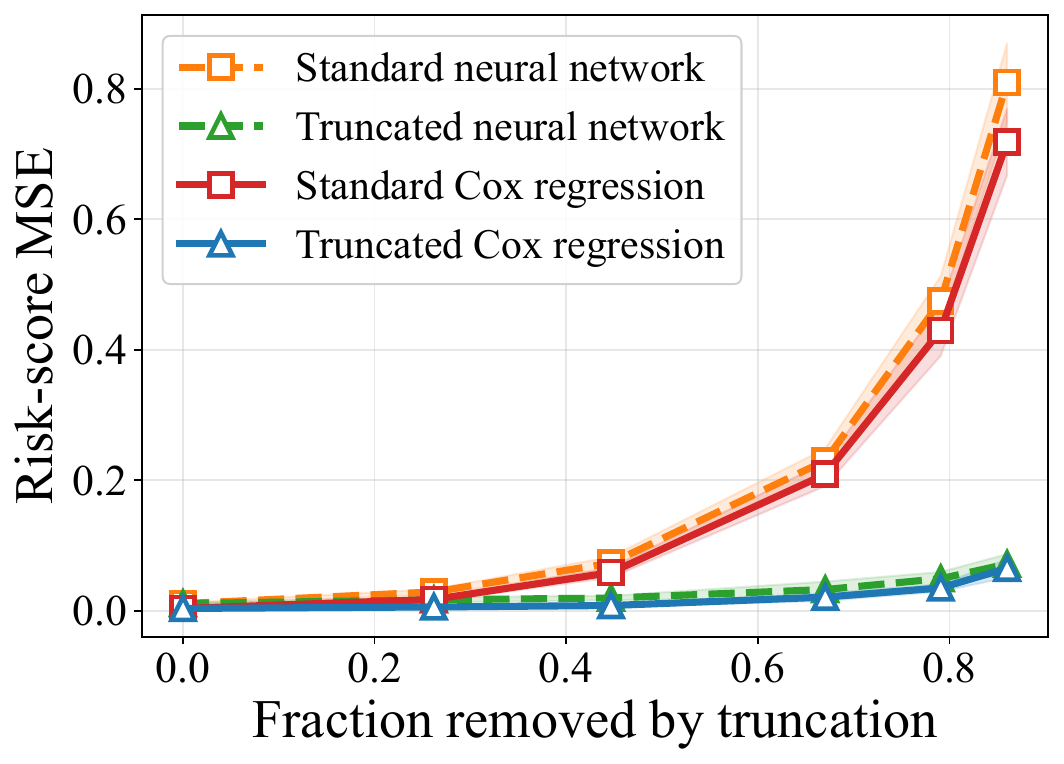}
        \caption{SUPPORT.}
        \label{fig:support}
    \end{subfigure}
    \hfill
    \begin{subfigure}[b]{0.24\textwidth}
        \centering
        \includegraphics[width=\linewidth]{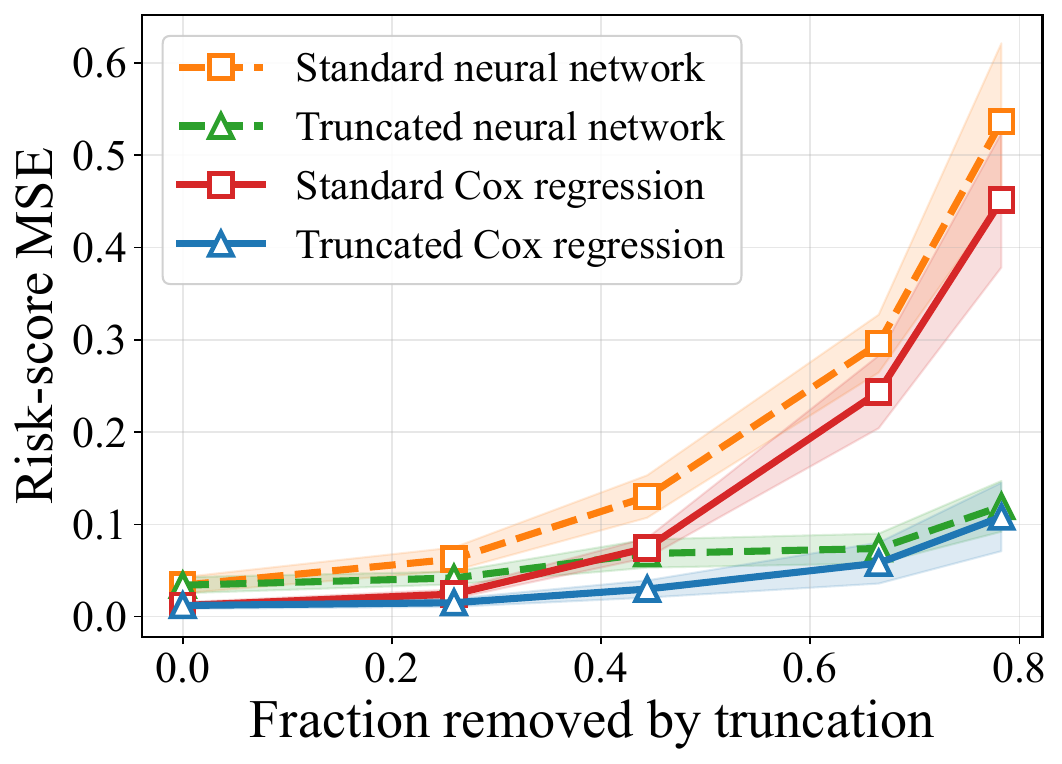}
        \caption{METABRIC.}
        \label{fig:metabric}
    \end{subfigure}
    \hfill
    \begin{subfigure}[b]{0.24\textwidth}
        \centering
        \includegraphics[width=\linewidth]{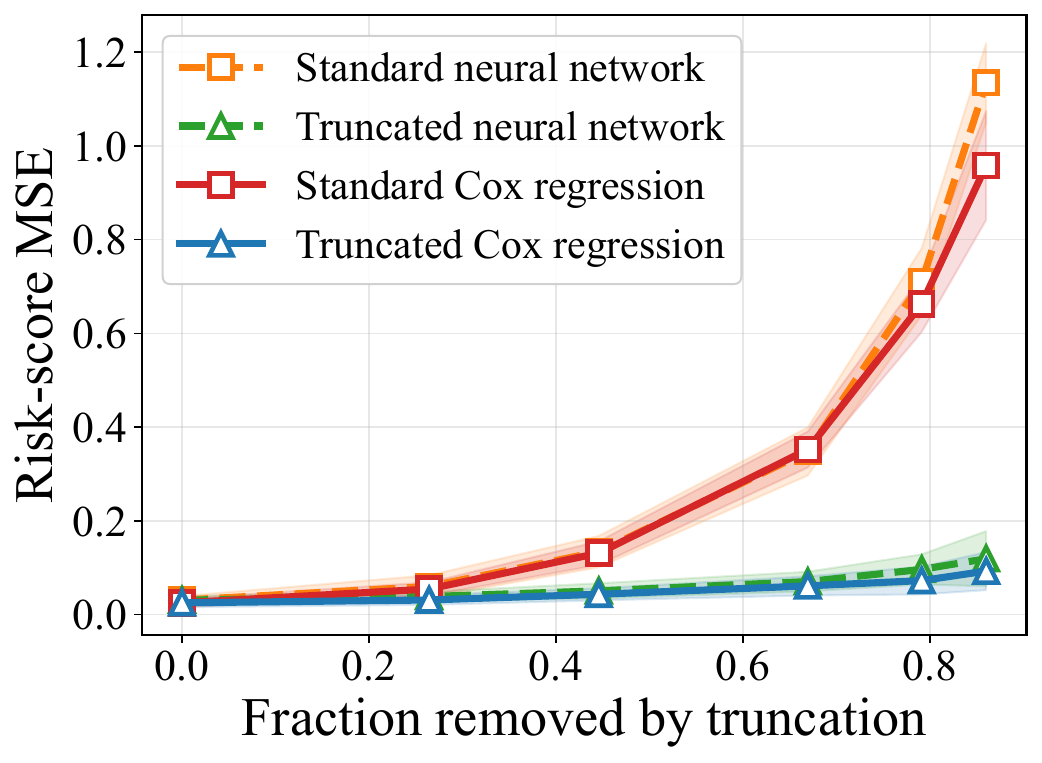}
        \caption{FLCHAIN.}
        \label{fig:flchain}
    \end{subfigure}
    
    \caption{Results of learning standard and truncated Cox regression and neural networks with first-order optimizers across datasets when the data-generating model is a Cox model. Shaded regions show standard deviation.}
    \label{fig:dataset-mse}
\end{figure*}

\begin{table*}[b!]
  \centering
  \resizebox{\linewidth}{!}{
  \begin{tabular}{l cc cc cc cc}
    \toprule
    & \multicolumn{2}{c}{METABRIC}
    & \multicolumn{2}{c}{SUPPORT}
    & \multicolumn{2}{c}{FLCHAIN}
    & \multicolumn{2}{c}{HEARTS} \\
    \cmidrule(lr){2-3}\cmidrule(lr){4-5}\cmidrule(lr){6-7}\cmidrule(lr){8-9}
    Method & Time (s) & MSE & Time (s) & MSE & Time (s) & MSE & Time (s) & MSE \\
    \midrule
    Adam \textcolor{gray}{(Ours)}
      & $0.79\pm0.38$ & $0.0542\pm0.0605$
      & $1.11\pm0.42$ & $0.0240\pm0.0263$
      & $1.13\pm0.77$ & $0.0625\pm0.0596$
      & $2.05\pm1.32$ & $0.0271\pm0.0197$ \\
    Projected SGD \textcolor{gray}{(Ours)}
      & $0.56\pm0.29$ & $0.0446\pm0.0474$
      & $0.86\pm0.33$ & $0.0231\pm0.0242$
      & $1.59\pm1.56$ & $0.0545\pm0.0433$
      & $5.15\pm3.58$ & $0.0316\pm0.0327$ \\
    EM
      & $0.16\pm0.10$ & $0.0462\pm0.0555$
      & $0.70\pm0.52$ & $0.0240\pm0.0277$
      & $1.40\pm0.55$ & $0.0877\pm0.0822$
      & $34.29\pm13.65$ & $0.0503\pm0.0475$ \\
    \bottomrule
  \end{tabular}}
   \caption{Computational efficiency of three optimizers for the same
    truncated Cox regression MLE across datasets with 10 independent runs per truncation level. Time is the mean wall-clock time to convergence (s) and MSE is the mean test risk-score error.}
  \label{tab:efficiency}
\end{table*}

We demonstrate that optimizing the truncated likelihood can be done with SGD in practical, semi-synthetic settings. We construct cohorts from four survival datasets: the \textit{United Network for Organ Sharing (UNOS)} registry for heart transplant graft survival which we refer to as HEARTS, SUPPORT for survival of hospitalized patients~\citep{Connors95:Controlled}, METABRIC for breast cancer survival~\citep{Curtis12:Genomic}, and FLCHAIN for mortality prediction from laboratory measurements~\citep{Dispenzieri12:Use}. For each dataset, we fit a survival model that defines the data-generating process, generate synthetic event times from this model, independently impose censoring at the same rate as in the original dataset, and then apply a known outcome-dependent truncation mechanism. We compare models trained on the resulting selected samples across increasing levels of truncation.

We first evaluate both standard Cox regression and the truncated Cox model from~\Cref{sec:prelims} optimized with projected SGD, as well as neural variants in which the linear risk score is replaced by a feed-forward neural network trained with Adam. This neural architecture is also known as \textit{DeepSurv}~\citep{Katzman18:Deepsurv}. Although our theoretical guarantees apply only to the linear Cox model, we also optimize the corresponding truncated likelihood for the neural network to test whether the same correction to the likelihood is useful empirically. We split each dataset into training and test sets and evaluate the \textit{mean squared error (MSE)} between learned and ground-truth risk scores on the untruncated test set. We consider both a linear Cox data-generating model (\Cref{fig:dataset-mse}) and a neural network data-generating model (\Cref{fig:dataset-mse-neural}). Full details of the data generation, truncation, censoring, and neural architecture are provided in~\Cref{sec:appendix-experimental-details}.

Across all four datasets, explicitly accounting for truncation substantially improves the accuracy as truncation becomes more severe. When truncation is ignored, the MSE increases rapidly with the amount of truncation. When the data-generating model is linear, truncated Cox regression closely recovers the true risk scores, while the neural model achieves similar performance. When the data-generating model is neural, the truncated neural network performs best.

These experiments complement our theoretical results by showing that the truncated Cox objective is practically optimizable by standard gradient-based methods. While our guarantees are specific to Cox regression, the same correction is effective for neural networks.

We also show the computational efficiency of first-order optimization compared to EM.  We implement a known-baseline variation of the truncated EM algorithm from~\citet{Rennert22:Cox} and compare its training time and MSE with projected SGD and Adam in~\Cref{tab:efficiency}. We include details of the EM algorithm in~\Cref{appendix:em}.  While our theoretical results only strictly apply to projected SGD, we show that, in practice, a common first order optimizer also converges quickly, namely Adam. On the three smaller datasets, METABRIC, SUPPORT, and FLCHAIN, all methods converge within 2 seconds and achieve comparable MSE. The difference becomes pronounced on HEARTS, which contains substantially more observations and covariates. EM requires over $30$ seconds to converge on average, compared with an average of $5.15$ seconds for projected SGD and $2.05$ seconds for Adam. EM is an order of magnitude slower than projected SGD and Adam. It also achieves worse MSE. To further push the scalability, we artificially enlarge the HEARTS dataset by generating Gaussian features and run our approaches on a semi-synthetic dataset with $N=750,000$ samples each with $2,000$ covariates. We refer to this dataset as HEARTS-Large (see \Cref{appendix:enlarged-hearts} for more details), and show the results in~\Cref{tab:scalability-aggregate}. A breakdown per truncation level is included in~\Cref{tab:scalability}. All algorithms achieve comparable MSE. However, Adam and projected SGD achieve this MSE around $100\times$ faster, converging in 20-30 seconds compared to roughly 45 minutes for EM.  These results highlight that direct first-order optimization scales significantly better than EM as the size and dimensionality of the dataset increases. 

\ifthenelse{\boolean{isSingleColumn}}{
    \renewcommand{\tabwidth}{0.6\linewidth}
}{
    \renewcommand{\tabwidth}{\linewidth}
}

\begin{table}[h!]
  \centering
\resizebox{\tabwidth}{!}{
  \begin{tabular}{l cc}
    \toprule
    Method & Time (s) & MSE \\
    \midrule
    Adam \textcolor{gray}{(Ours)}          & $25.4\pm20.3$      & $0.0281\pm0.0283$ \\
    Projected SGD \textcolor{gray}{(Ours)} & $22.0\pm13.4$      & $0.0276\pm0.0276$ \\
    EM                                     & $2607.1\pm1210.6$  & $0.0272\pm0.0271$ \\
    \bottomrule
  \end{tabular}}
  \caption{Average optimizer efficiency on  HEARTS-Large across truncation levels.}
  \label{tab:scalability-aggregate}
\end{table}
\section{Conclusions}
\label{sec:conclusion}

We analyzed survival analysis under outcome-dependent truncation and right censoring. We characterized when the parameters of Cox regression can be learned efficiently. When the baseline hazard is known, projected SGD converges efficiently to the true parameter values. We established negative results when the baseline hazard is unknown showing that a favorable optimization landscape does not always exist. 

Semi-synthetic experiments on four datasets demonstrated that first-order methods can effectively optimize truncated Cox likelihoods, reducing prediction error relative to models that ignore selection. The same correction also improved neural networks empirically. We also showed orders of magnitude speedups relative to EM. As survival models increasingly inform high-stakes decision making in medicine, efficiently learning a survival model under truncation is an important capability to scale models to higher dimensions.
\section*{Acknowledgments}

Tuomas Sandholm and his PhD students Ioannis Anagnostides and Itai Zilberstein are supported by NIH award A240108S001, the Vannevar Bush Faculty Fellowship ONR N00014-23-1-2876, and National Science Foundation grant RI-2312342. Itai Zilberstein is also supported by the NSF Graduate Research Fellowship Program under grants DGE2140739 \& DGE2631988. Arman Kilic is supported by NIH RO1 grant 5R01HL162882-03 which contributed to the funding for completion of this project. Arman Kilic is a speaker and consultant for Abiomed, Abbott, 3ive, and LivaNova, and founder and owner of QImetrix. All additional authors have no financial relationships to disclose. Any opinions, findings, and conclusions or recommendations expressed in this material are those of the author(s) and do not necessarily reflect the views of the funding agencies. 

\bibliography{references}

\clearpage

\appendix
\section{Additional related work}
\label{sec:related}

Our work sits at the intersection of two prevalent areas of research: efficient learning under truncation and survival analysis. 

Learning from truncated samples has been studied in regression settings without censoring~\citep{Ilyas20:Theoretical,Daskalakis19:Computationally}. These prior works proved that the parameters of standard regression methods, such as probit and logistic regression, can be optimized efficiently with first-order methods~\citep{Ilyas20:Theoretical}. More recent work gives computationally and statistically efficient algorithms for learning exponential families from truncated samples under general truncation sets~\citep{Lee23:Learning}. This work is closely related to our setting.  
When transforming event times by the cumulative baseline hazard, an uncensored event time follows an exponential distribution. Our setting differs in two challenging ways. The truncated Cox likelihood is generally not convex in the regression coefficients because the exponential rate is nonlinear in the parameters. We also consider right censoring, where censored observations reveal only a lower bound on the event time rather than a sample from the truncated exponential distribution. We show that, despite these challenges, the population likelihood has a unique stationary point and gradient descent moves toward the true parameter value. A separate line of work studies the harder setting in which the truncation mechanism itself is unknown~\citep{Lee24:Efficient,Kouridakis26:Linear}.

In the context of survival analysis, there is a large body of work on predicting time-to-event under the typical i.i.d. assumption and right censoring~\citep{Ishwaran08:Random,Kaplan58:Nonparametric,Wei92:Accelerated} including neural architectures~\citep{Lee18:Deephit,Nagpal21:Deep,Katzman18:Deepsurv,Zilberstein26:Aligning}. Despite these modern predictors, Cox regression~\citep{Cox72:Regression} remains one of the most widely used models. The predictions from a survival model are often used to inform decision making in high-stakes settings. One of the most prominent applications is organ allocation~\citep{Papalexopoulos24:Reshaping,OPTN25:ContinuousDistribution,Dickerson15:Futurematch,Zilberstein26:Near,Zilberstein26:Learning,Zilberstein26:Preference,Anagnostides26:Position,Berrevoets20:OrganITE}. 

Truncation in survival analysis has a long history, including classical methods for left-truncated and right-censored data~\citep{Tsai87:Note}, length-biased sampling~\citep{Wang96:Hazards,Qin10:Statistical}, and more general biased sampling schemes~\citep{Kim13:Unified,Yu16:Outcome}. More recent work focuses on Cox regression for doubly truncated data using inverse-probability weighting, pseudo-likelihood, and conditional likelihood~\citep{Mandel18:Inverse,Rennert18:Cox,Shen19:Pseudo,Rennert22:Cox,Shen20:Conditional,Chiou19:Transformation,Li20:Adaptive}. These approaches establish consistency and asymptotic properties, but do not characterize the global optimization landscape. Some likelihood-based approaches instead rely on iterative procedures such as EM with only local convergence guarantees. Another related line of work estimates causal treatment effects under selection bias, including truncation~\citep{Zhang20:Learning,Alaa17:Bayesian}.

Classical impossibility results in survival analysis establish nonidentifiability under partial observation, most notably for competing risks~\citep{Tsiatis75:Nonidentifiability}. In contrast to these results, we show  that even when the model is identifiable and correctly specified, learning an unknown baseline can produce a population likelihood with suboptimal local minima.

From a broader context, our work contributes to the study of why gradient-based methods succeed in nonconvex optimization, including results on avoiding strict saddle points~\citep{Ge15:Escaping,Lee16:Gradient,Lee19:First} and establishing that all local minima are globally optimal in structured problems~\citep{Ge16:Matrix}. In our paper, we establish a favorable population landscape for truncated and right-censored Cox regression with a known baseline hazard: despite nonconvexity, we show that the true parameter is the unique stationary point.





\section{Further background on Cox regression}
\label{sec:appendix-prelims}

We derive the maximum likelihood estimator for the Cox model. For independent observations, the contribution to the likelihood from sample $i$ is
\[ \big(h_i(T_i)\big)^{\delta_i} S_i(T_i)
=\big(h_0(T_i)e^{\vx_i^{\top}\vtheta}\big)^{\delta_i} \exp\big(-e^{\vx_i^{\top}\vtheta} H_0(T_i)\big). 
\]
As a result, the full likelihood reads
\begin{align*}
\cL(\vtheta)
&=\prod_{i=1}^n ( h_0(T_i)e^{\vx_i^{\top}\vtheta})^{\delta_i} \exp (-e^{\vx_i^{\top}\vtheta} H_0(T_i) ) \\
&=\left( \prod_{i=1}^n h_0(T_i)^{\delta_i} \right) \exp \left(\sum_{i=1}^n \delta_i \vx_i^{\top}\vtheta - \sum_{i=1}^n e^{\vx_i^{\top}\vtheta} H_0(T_i) \right).
\end{align*}

When the baseline hazard $h_0(\cdot)$ is known, the first term in the right-hand side, $\prod_{i=1}^n h_0(T_i)^{\delta_i}$, is independent of $\vtheta$. As a result, the log-likelihood---up to an additive constant---reads
\begin{equation}
\label{eq:loglik}
\ell(\vtheta)=\sum_{i=1}^n \delta_i \vx_i^{\top}\vtheta - \sum_{i=1}^n e^{\vx_i^{\top}\vtheta} H_0(T_i).
\end{equation}
We now observe that the log-likelihood is concave, making it amenable to standard optimization algorithms. Indeed, the gradient with respect to $\vtheta$ is
\begin{align*}
U(\vtheta) &= \nabla_{\vtheta} \ell(\vtheta) = \sum_{i=1}^n \delta_i\vx_i - \sum_{i=1}^n H_0(T_i) e^{\vx_i^{\top}\vtheta} \vx_i \\
&= \sum_{i=1}^n \big(\delta_i - H_0(T_i) e^{\vx_i^{\top}\vtheta}\big) \vx_i,
\end{align*}
so the Hessian can be expressed as
\begin{align*}
\nabla_{\vtheta}^2 \ell(\vtheta) &= -\sum_{i=1}^n H_0(T_i) e^{\vx_i^{\top}\vtheta} \vx_i\vx_i^{\top}.
\end{align*}
The Hessian is negative semidefinite, so the log-likelihood is concave. 

We now derive the likelihood function when we have interval truncation.  Recall that $\pi_i(\vtheta)$ denotes the probability that sample $i$ is included, while $\pi_i(t\mid\vtheta)$ denotes the joint probability that it is included and that $T_i^*>t$.

For sample $i$ with observed data $(T_i, \delta_i)$, we divide into two cases to capture the contribution to the likelihood. First, if $\delta_i=1$, then $T_i^*=T_i$, and the contribution to the likelihood is
\[ \cL_i(\vtheta) \propto \frac{h_i(T_i)S_i(T_i)}{\pi_i(\vtheta)}, \]
where we omit the factor $\rho(T_i)$, which is independent of $\vtheta$. Similarly, if $\delta_i=0$, then $T_i^*>T_i$, and the contribution is
\[ \cL_i(\vtheta) \propto \frac{\pi_i(T_i \mid \vtheta)}{\pi_i(\vtheta)}. \]

The full likelihood, up to factors independent of
$\vtheta$, is 
\[ \cL(\vtheta) \propto \prod_{i=1}^n \frac{ [h_i(T_i)S_i(T_i)]^{\delta_i} [\pi_i(T_i \mid \vtheta)]^{1 -\delta_i}}{\pi_i(\vtheta)}. \]



\section{Omitted proofs}

\TruncExpMean*

\begin{proof}
Let $p_\lambda$ denote the density of the weighted truncated
exponential distribution with rate $\lambda$, 
\[ p_\lambda(u) = \frac{\lambda e^{-\lambda u} \sum_{k=1}^K q_k\mathbbm{1}\{A_k<u<B_k\}}{\sum_{k=1}^K q_k \left(e^{-\lambda A_k}-e^{-\lambda B_k}\right)}. \]

Taking the expectation, we get
\[ \mu(\lambda) = \frac{ \sum_{k=1}^K q_k \int_{A_k}^{B_k}\lambda u e^{-\lambda u}du}{\sum_{k=1}^K q_k\left(e^{-\lambda A_k}-e^{-\lambda B_k}\right)}. \]

For each interval, we have
\begin{align*}
\int_{A_k}^{B_k}\lambda u e^{-\lambda u} du
&= \left[-u e^{-\lambda u}\right]_{A_k}^{B_k} + \left[-\frac{1}{\lambda}e^{-\lambda u}\right]_{A_k}^{B_k} \\
&= A_k e^{-\lambda A_k}-B_k e^{-\lambda B_k} + \frac{1}{\lambda} \left(e^{-\lambda A_k}-e^{-\lambda B_k}\right).
\end{align*}
The claim follows by substituting the above into the expectation.
\end{proof}

\TruncExpMonotonicity*

\begin{proof}
Let $p_{\lambda_1}$ and $p_{\lambda_2}$ denote the respective densities, and define
\[ \omega(u)=e^{-(\lambda_2-\lambda_1)u}. \]
The intervals and retention probabilities are the same for both distributions, therefore the densities satisfy
\[ p_{\lambda_2}(u) = \frac{\omega(u)p_{\lambda_1}(u)}{\E[\omega(U_1)]}. \]
Consequently,
\begin{align*}
\E_{\lambda_2}[\varphi(U_2)]-\E_{\lambda_1}[\varphi(U_1)]
&= \frac{\E_{\lambda_1}[\varphi(U_1)\omega(U_1)]}{\E_{\lambda_1}[\omega(U_1)]} - \E_{\lambda_1}[\varphi(U_1)] \\
&= \frac{\operatorname{Cov}(\varphi(U_1),\omega(U_1))}{\E_{\lambda_1}[\omega(U_1)]}.
\end{align*}

Let $U_1'$ be an independent copy of $U_1$. Then
\[ 2\operatorname{Cov}_{\lambda_1}(\varphi(U_1),\omega(U_1)) = \E_{\lambda_1}\left[ \bigl(\varphi(U_1)-\varphi(U_1')\bigr) \bigl(\omega(U_1)-\omega(U_1')\bigr) \right]. \]
Because $\varphi$ is nondecreasing and $\omega$ is strictly decreasing, the integrand is nonpositive, meaning the covariance is nonpositive. This proves the first claim.

To prove the second claim, if $\varphi(U_1)$ is not almost surely constant, then
\[ \pr\bigl(\varphi(U_1) \ne \varphi(U_1')\bigr) > 0. \]
In this case, the integrand is strictly negative because $\omega$ is strictly decreasing, implying the covariance is strictly negative, giving
\[ \E_{\lambda_1}[\varphi(U_1)] > \E_{\lambda_2}[\varphi(U_2)]. \]
\end{proof}

\TruncExpCensoring*

\begin{proof}
Since $\mu_c(\lambda)=\E[U\mid U>c]>c$, $m_{\lambda,c}$ increases on $[0,c]$, jumps upward immediately after $c$, and is constant on the interval $(c,\infty)$. This proves it is nondecreasing.

We prove the second statement by expanding the conditional mean,
\begin{align*}
\E[m_{\lambda,c}(U)]
&= \E[U\mathbbm{1}\{U\le c\}] + \mu_c(\lambda)\pr(U>c) \\
&= \E[U\mathbbm{1}\{U\le c\}] + \E[U\mid U>c]\pr(U>c) \\
&= \E[U\mathbbm{1}\{U\le c\}] + \E[U\mathbbm{1}\{U>c\}] \\
&= \E[U] = \mu(\lambda).
\end{align*}

Finally, suppose that $\pr(U\le c)>0$. Since $U$ has a density supported on the positive-weight intervals, there is some $k$ with $q_k>0$ and $A_k<c$. Therefore, we can choose $a,b\in(A_k,B_k)$ with $a<b<c$. On these points,
\[ m_{\lambda,c}(a)=a<b=m_{\lambda,c}(b), \]
so $m_{\lambda,c}$ is nonconstant on $\mathcal S$.
\end{proof}

\CoxTruncExp*

\begin{proof}
The density of $T_i^*$ is
\[ 
p_{T_i^*}(t) = \frac{ \rho(t)h_i(t;\vtheta)S_i(t;\vtheta)}{\pi_i(\vtheta)} =\frac{\rho(t) h_0(t)e^{\vx_i^\top\vtheta}\exp\left(-e^{\vx_i^\top\vtheta} H_0(t) \right)}{\pi_i(\vtheta)}.
\]
We perform a change of variable, defining $u= H_0(t)$. This gives $du= h_0(t)dt$. Since $H_0$ is strictly increasing,
\[t_{0,k}<t<t_{1,k} \Longleftrightarrow  H_0(t_{0,k})< u <  H_0(t_{1,k}). \]
Therefore, we can write the density of $H_0(T_i^*)$ for $u\ge0$ as
\[ 
p_{ H_0(T_i^*)}(u) = \frac{ e^{\vx_i^\top\vtheta} \exp\left(-e^{\vx_i^\top\vtheta}u\right) \sum_{k=1}^K q_k \mathbbm{1}\{ H_0(t_{0,k})<u< H_0(t_{1,k}) \}}{ \pi_i(\vtheta) }.
\]
We complete the proof by expanding and substituting the denominator,
\begin{align*}
\pi_i(\vtheta) &=
\sum_{k=1}^K q_k \bigl[ S_i(t_{0,k};\vtheta) - S_i(t_{1,k};\vtheta) \bigr] \\
&= \sum_{k=1}^K q_k  \left[ \exp\left(
-e^{\vx_i^\top\vtheta} H_0(t_{0,k}) \right) -  \exp\left(-e^{\vx_i^\top\vtheta} H_0(t_{1,k}) \right) \right].
\end{align*}

The resulting density defines a truncated exponential with parameters $\lambda = e^{\vx_i^\top\vtheta}$ and  $([A_k, B_k], q_k) = \left([ H_0(t_{0,k}), H_0(t_{1,k})],q_k \right)$ for $k=\{1, \dots, K\}$.
\end{proof}

\CoxLikelihoodGrad*

\begin{proof}
Dropping terms independent of $\vtheta$, the single-observation negative log-likelihood is
\[
\losslike(\vtheta) = \delta_i \left[ -\vx_i^\top\vtheta
+ \expshort H_0(T_i) \right] - (1-\delta_i)\log\pi_i(T_i\mid\vtheta)
+ \log\pi_i(\vtheta).
\]

We first differentiate $\pi_i(t \mid \vtheta)$. For every fixed $t$ with $\pi_i(t\mid\vtheta)>0$,
\[ \pi_i(t\mid\vtheta) = \int_t^\infty \rho(s)h_i(s;\vtheta)S_i(s;\vtheta)ds. \]
The underlying density satisfies
\[ 
\nabla_{\vtheta} \left[h_i(s;\vtheta) S_i(s;\vtheta)\right] = \left[1-e^{\vx_i^\top\vtheta}H_0(s)\right] h_i(s;\vtheta)S_i(s;\vtheta)\vx_i.
\]
Since $\rho$ does not depend on $\vtheta$, it follows that
\begin{align*}
\nabla_{\vtheta}\pi_i(t\mid\vtheta) &= \left[ \pi_i(t\mid\vtheta) - \expshort \int_t^\infty H_0(s)\rho(s)h_i(s;\vtheta)S_i(s;\vtheta)ds \right]\vx_i \\
&= \pi_i(t\mid\vtheta) \left[ 1-\expshort\mu_i(t\mid\vtheta)\right]\vx_i,
\end{align*}
where we used
\[ \mu_i(t\mid\vtheta) = \frac{ \displaystyle
\int_t^\infty H_0(s)\rho(s)h_i(s;\vtheta)S_i(s;\vtheta)ds}{\pi_i(t\mid\vtheta)}. \]
Dividing by $\pi_i(t\mid\vtheta)$ gives
\[ \nabla_{\vtheta}\log\pi_i(t\mid\vtheta) = \left[ 1-\expshort\mu_i(t\mid\vtheta) \right]\vx_i. \]

We complete the proof by substituting the logarithmic derivative into the expression
for $\losslike$.
\end{proof}

\CoxStatPoint*

\begin{proof}
Write $\lambda=e^z$ and $\lambda^*=e^{z^*}$. By~\Cref{lem:cox-truncated-exponential},  $U:=H_0(T^*)$
follows a weighted truncated exponential distribution with rate $\lambda^*$. We write $\E_\lambda$ for expectation over the joint distribution of $U$ and $C$, where $U$ follows the weighted truncated exponential distribution with rate $\lambda$ and $C$ follows the fixed censoring distribution, independently of $U$. $\E_{\lambda^*}$ and $\E_\lambda$ differ only in the distribution of $U$. Let $\mu(\lambda)=\E_\lambda[U]$.

For a censoring threshold $c$, define $m_{\lambda,c}$ to be the function from~\Cref{lem:truncated-exponential-censoring}. By~\Cref{lem:cox-likelihood-gradient}, taking the expectation of one observation's gradient gives
\[ \losspop_{\vx}'(z) = \lambda\left( \E_{\lambda^*} \left[m_{\lambda,H_0(C)}(U) \right] - \mu(\lambda) \right). \]
Since $C$ is independent of $U$ and its distribution does not depend on $\lambda$, applying~\Cref{lem:truncated-exponential-censoring} gives
\[ \mu(\lambda) = \E_{\lambda} \left[ m_{\lambda,H_0(C)}(U) \right]. \]
Dividing $\losspop_{\vx}'(z)$ by $\lambda$ and substituting $\mu(\lambda)$ gives,
\begin{equation}
\label{eq:population-gradient-comparison}
\frac{\losspop_{\vx}'(z)}{\lambda} = \E_{\lambda^*} \left[m_{\lambda,H_0(C)}(U)\right] - \E_{\lambda} \left[m_{\lambda,H_0(C)}(U)\right].
\end{equation}

Now we consider three cases comparing $\lambda$ to $\lambda^*$. Suppose first that $\lambda>\lambda^*$. For every fixed censoring time $c$,~\Cref{lem:truncated-exponential-censoring} tells us that the function $m_{\lambda,H_0(c)}$ is
nondecreasing. Holding this function fixed,~\Cref{lem:truncated-exponential-monotonicity} gives
\[ \E_{\lambda^*} \left[m_{\lambda,H_0(c)}(U)\right] \ge \E_{\lambda} \left[m_{\lambda,H_0(c)}(U)\right]. \]

Since we assume $\pr(\delta=1)>0$,
\[ \E_C\left[ \pr_{\lambda^*}(U\le H_0(C)) \right]>0. \]
Therefore, with positive probability, the censoring time  falls into a set where $m_{\lambda,H_0(c)}$ is nonconstant. On this set, the inequality above is strict, giving
\[ \losspop_{\vx}'(z)>0 \text{ whenever }\lambda>\lambda^*. \]

If $\lambda<\lambda^*$, the same argument with the two rates reversed gives $F_{\vx}'(z)<0$. Finally, when $\lambda=\lambda^*$, the two expectations in \eqref{eq:population-gradient-comparison} coincide, so $\losspop_{\vx}'(z^*)=0$.

Since $\lambda=e^z$ is strictly increasing, 
\[ (z-z^*)\losspop_{\vx}'(z)>0 \text{ for every }z\ne z^*.\]
Therefore, $z^*$ is the unique stationary point. Since $\losspop_{\vx}$ is strictly decreasing to the left of $z^*$ and strictly increasing to its right, $z^*$ is also the unique global minimizer.
\end{proof}

\CoxPopUnique*

\begin{proof}
For a given covariate vector $\vx$, let $z=\vx^\top\vtheta$ and $z^*=\vx^\top\vtheta^*$. Since the likelihood depends on $\vtheta$ only through $z$, we write
\[ \losspop_{\vx}(z) = \E_{\vtheta^*}[-\ell(z;T,\delta) \mid \vx], \]
where the expectation is under the true distribution.
By \Cref{thm:cox-conditional-population},
\begin{equation}\label{eq:sign_matching}
(z-z^*)\losspop_{\vx}'(z)\ge0,
\end{equation}
with strict inequality whenever $z\ne z^*$, and $\losspop_{\vx}'(z^*)=0$.

To analyze the joint population gradient, we use the chain rule and differentiate under the expectation to obtain
\[ \nabla_{\vtheta}\losspop(\vtheta) = \E_{\vx} \left[\losspop_{\vx}'(\vx^\top\vtheta)\vx \right]. \]

Taking the inner product with $\vtheta-\vtheta^*$ gives
\begin{align*}
\left\langle \nabla_{\vtheta}\losspop(\vtheta), \vtheta-\vtheta^* \right\rangle
&= \E \left[\losspop_{\vx}'(z)\vx^\top(\vtheta-\vtheta^*) \right]\\
&= \E \left[\losspop_{\vx}'(z)(z-z^*)\right].
\end{align*}
By \eqref{eq:sign_matching}, the integrand is nonnegative. Since $\mathbb{E}[\vx \vx^\top] \succ 0$, for any $\vtheta \ne \vtheta^*$, the event $z\ne z^*$ has strictly positive probability. Therefore,
\begin{equation*}
\left\langle \nabla_{\vtheta}\losspop(\vtheta),  (\vtheta-\vtheta^* ) \right\rangle>0 \text{ for all } \vtheta \ne \vtheta^*.
\end{equation*}
Hence, no incorrect parameter is stationary. At the true parameter,
$\losspop_{\vx}'(z^*)=0$ for almost every $\vx$, and therefore $\nabla_{\vtheta}\losspop(\vtheta^*)=\mathbf{0}$. So, $\vtheta^*$ is the unique stationary point.
\end{proof}

\CoxSGDFiniteTime*

\begin{proof}
We first obtain quantitative bounds on the population gradient. The derivative of a weighted truncated exponential mean with respect to its rate is the negative of its variance. Differentiating \Cref{lem:cox-likelihood-gradient} at the truth and applying the law of total variance therefore gives
\[ \nabla^2\losspop(\vtheta^*)=\E_{\vtheta^*}\left[e^{2\vx^\top\vtheta^*}\operatorname{Var}_{\vtheta^*}\left(\E_{\vtheta^*}[H_0(T^*)\mid\vx,T,\delta]\mid\vx\right)\vx\vx^\top\right], \]
where all expectations are under the true sampled distribution. For almost every $\vx$, uncensored events have positive probability and reveal a nonconstant, continuously distributed transformed event time. The conditional variance is therefore positive. Covariate nondegeneracy then implies $\nabla^2\losspop(\vtheta^*)\succ0$.

Since $\nabla\losspop(\vtheta^*)=\mathbf0$, differentiability of the gradient at $\vtheta^*$ gives constants $0<\kappa\le L<\infty$ such that
\[ \left\langle\nabla\losspop(\vtheta),\vtheta-\vtheta^*\right\rangle\ge\kappa\|\vtheta-\vtheta^*\|_2^2 \text{ and } \|\nabla\losspop(\vtheta)\|_2\le L\|\vtheta-\vtheta^*\|_2 \]
in a neighborhood of $\vtheta^*$. Outside a sufficiently small open ball around $\vtheta^*$, the ratios of the left-hand sides to $\|\vtheta-\vtheta^*\|_2^2$ and $\|\vtheta-\vtheta^*\|_2$, respectively, are continuous on the remaining compact subset of $\Theta$. The first ratio is strictly positive by \Cref{thm:cox-population-uniqueness}. Adjusting $\kappa$ and $L$ therefore extends both bounds to all of $\Theta$.

Let $f_{i,j}$ denote the single-observation negative log-likelihood for $(\vx_{i,j},T_{i,j},\delta_{i,j})$. Each minibatch is independent of the preceding iterates. Conditional on $\vtheta^{(i)}$, its gradients are therefore independent and unbiased, so
\[ \E\left[\left\|\frac1b\sum_{j=1}^b\nabla f_{i,j}(\vtheta^{(i)})-\nabla\losspop(\vtheta^{(i)})\right\|_2^2\mid\vtheta^{(i)}\right]\le\frac{B^2}{b}. \]
Expanding the squared distance after projection and applying the preceding bounds gives
\[ \E\left[\|\vtheta^{(i+1)}-\vtheta^*\|_2^2\mid\vtheta^{(i)}\right]\le(1-2\kappa\eta_i+L^2\eta_i^2)\|\vtheta^{(i)}-\vtheta^*\|_2^2+\frac{\eta_i^2B^2}{b}. \]
The step-size choice ensures $\eta_i\le\kappa/L^2$. Taking expectations thus yields
\[ \E\left[\|\vtheta^{(i+1)}-\vtheta^*\|_2^2\right]\le\left(1-\frac{2}{i+t_0}\right)\E\left[\|\vtheta^{(i)}-\vtheta^*\|_2^2\right]+\frac{4B^2}{\kappa^2b(i+t_0)^2}. \]
Multiplying by $(i+t_0)(i+t_0-1)$, using $(i+t_0-1)/(i+t_0)\le1$, and summing over $i=0,\ldots,t-1$ cancels consecutive terms and gives
\[ (t+t_0-2)(t+t_0-1)\E\left[\|\vtheta^{(t)}-\vtheta^*\|_2^2\right]\le(t_0-2)(t_0-1)\|\vtheta^{(0)}-\vtheta^*\|_2^2+\frac{4B^2t}{\kappa^2b}. \]
Since $(t+t_0-2)(t+t_0-1)\ge t^2$ and $(t_0-2)(t_0-1)\le t_0^2$, the claimed bound follows.
\end{proof}

\CoxUnknownBaselineLocalMin*

\begin{proof}


We start by defining a Cox model with one regression coefficient $\theta$ and one unknown baseline parameter $\gamma$. Let $\tau=1/2$, let the covariate $x\in\{0,1\}$ have equal probability in the sampled population, and retain individuals exactly when $0<T^*<\tau$. Assume no censoring. Define the hazard as $h(t\mid x;\vbeta)=e^{x\theta}g_\gamma(t)$, where $\vbeta=(\theta,\gamma) \in (0,\infty)\times\mathbb{R}$. We will construct the baseline family $g_\gamma$ using fixed constants ${\lambda'},a,M$. These constants are \textit{knobs} that we will tune to get the desired outcomes from our construction. The construction proceeds in three main steps: (1) choose ${\lambda'}$ to eliminate the regression gradient; (2) choose $a$ to eliminate the baseline-shape gradient; and (3) choose $M$ to make the Hessian positive definite. 

We construct the baseline family to satisfy $g_1(t)=1/(1-t)$ and $g_0(t)=1$ on $[0,\tau]$. The true parameter $\vbeta^*=(\log3,1)$ then generates the sampled densities
\[ p_0^*(t)=2 \text{ and } p_1^*(t)=\frac{24}{7}(1-t)^2 \text{ for } t \in (0,\tau). \]
Denote $\E_x^*$ as the expectation under $p_x^*$ and let $m_x=\E_x^*[T]$. Direct integration gives $m_0=1/4$, $m_1=11/56$, $\E_0^*[T^2]=1/12$, and $\E_1^*[T^2]=2/35$.

At the incorrect parameter value $\gamma=0$, the sampled distribution in group $x$ is exponential with rate $e^{x\theta}$, truncated to $(0,\tau)$. Since $\theta$ affects only group $x=1$, we choose its rate to match that group's true mean. Let ${\lambda'}>0$ be the unique solution of
\[ \mu({\lambda'})=\frac{11}{56} \text{ where } \mu(\lambda):=\frac1\lambda-\frac{1}{2(e^{\lambda/2}-1)}. \]
The function $\mu(\lambda)$ computes the mean of an exponential distribution of rate $\lambda$ truncated to $(0,\tau)$. This solution exists because $\mu$ is continuous and strictly decreasing from $1/4$ to zero. The bounds $e>8/3$ and $e^{3/2}<9/2$ give
\[ \mu(2)>\frac15>\frac{11}{56}>\frac4{21}>\mu(3), \]
so $2<{\lambda'}<3$. For completeness, the upper bound follows from
\[ e^{3/2}\le \sum_{j=0}^3\frac{(3/2)^j}{j!}+\frac{(3/2)^4}{4!}\frac{1}{1-3/10}=\frac{2011}{448}<\frac92. \]
Define the false parameter $\vbeta'=(\log{\lambda'},0)$. In particular, $\theta'=\log{\lambda'}<\log3=\theta^*$.

Mean matching will make the regression derivative vanish at $\vbeta'$. We next choose how the baseline varies so that the derivative of the population loss with respect to $\gamma$ also vanishes there. Writing $\lambda'_x=(\lambda')^x$, define, for $j=1,2$,
\[ A_j:=\frac12\sum_{x=0}^1\left[-\E_x^*[T^j]+\frac{\lambda'_x}{j+1}\E_x^*[T^{j+1}]+\frac{\lambda'_x\tau^{j+1}}{(j+1)(e^{\lambda'_x\tau}-1)}\right]. \]
The constant $A_j$ is the derivative of the population loss at $s=0$ when the constant baseline is replaced by $\exp(s t^j)$ on the retained interval, with the rates for $x=0$ and $x=1$ fixed at $1$ and ${\lambda'}$.

Using the moments above and the identity ${\lambda'}\tau/(e^{{\lambda'}\tau}-1)=1-{\lambda'} m_1$ gives
\[ A_1=\frac12\left[-\frac5{24}+\frac{1}{8(e^{1/2}-1)}+\frac3{56}-\frac{23{\lambda'}}{1120}\right]<0. \]

Indeed, $e^{1/2}>1+1/2+1/8+1/48=79/48$ implies
\[ -\frac5{24}+\frac{1}{8(e^{1/2}-1)}<-\frac{11}{744}<-\frac1{80}, \]
while ${\lambda'}>2$ gives $3/56-23{\lambda'}/1120<1/80$. Define the constant $a:=A_2/A_1$ which will be used to cancel the derivative of the baseline-shape, while the sign $A_1<0$ will let us increase the curvature around $\vbeta'$.

To define the full Cox model, we will define $g_\gamma$. Fix a smooth function $\chi:[0,\infty)\to[0,1]$ equal to one on $[0,\tau]$ and zero on $[3/4,\infty)$. For a constant $M\ge0$ to be chosen below, define
\[ g_\gamma(t):=\exp\left\{\chi(t)\left[-\gamma^3\log(1-t)+(\gamma-\gamma^3)(t^2-at)+\frac{M}{2}(\gamma^3-\gamma^2)t\right]\right\} \text{ for } t \in [0,3/4), \]
and set $g_\gamma(t)=1$ for $t\ge3/4$. This family is smooth and strictly positive, satisfies the required identities at $\gamma=0,1$, and has cumulative hazard $G_\gamma(t):=\int_0^t g_\gamma(u)du$ diverging as $t\to\infty$. The cutoff only ensures that the baseline is well-defined outside the retained interval; it does not affect the sampled distributions.

The sampled density and population negative log-likelihood of our Cox model are
\[ p_x(t;\vbeta)=\frac{e^{x\theta}g_\gamma(t)\exp[-e^{x\theta}G_\gamma(t)]}{1-\exp[-e^{x\theta}G_\gamma(\tau)]} \text{ for } t \in (0, \tau), \]
\[ \losspop_M(\vbeta):=-\frac12\sum_{x=0}^1\int_0^\tau p_x^*(t)\log p_x(t;\vbeta)dt. \]
The subscript denotes the dependence on the constant $M$.

At $\vbeta'$, differentiation with respect to $\theta$ gives
\[ \partial_\theta \losspop_M(\vbeta')=\frac12{\lambda'}\bigl(m_1-\mu({\lambda'})\bigr)=0. \]
Therefore, matching the mean when $x=1$ makes the regression derivative vanish. On the retained interval, $\partial_\gamma\log g_\gamma(t)|_{\gamma=0}=t^2-at$. Differentiating the negative log-density therefore gives
\[ \partial_\gamma \losspop_M(\vbeta')=A_2-aA_1=0. \]
Hence $\vbeta'$ is stationary for every $M$.

Stationarity alone does not rule out a saddle point. We now use $M$ to make the loss increase in every direction around $\vbeta'$. Let $D,w,b$ be the entries of the Hessian when $M=0$,
\[ \nabla^2\losspop_0(\vbeta')=\begin{pmatrix}D&w\\w&b\end{pmatrix}. \]
The regression entry satisfies
\[ D=\frac12(\lambda')^2\operatorname{Var}_{\lambda'}(T)>0, \]
where the variance is under the exponential distribution of rate $\lambda'$ truncated to $(0,\tau)$. This follows from $\mu'(\lambda)=-\operatorname{Var}_\lambda(T)$ and the equality of the means.

Changing $M$ leaves the value and first shape derivative of $g_\gamma$ at $\gamma=0$ unchanged, so it changes neither stationarity nor the Hessian entries $D,w$. However,
\[ \left.\partial_{\gamma\gamma}\log g_\gamma(t)\right|_{\gamma=0}=-Mt, \text{ and } \left.\partial_{\gamma\gamma}g_\gamma(t)\right|_{\gamma=0}=(t^2-at)^2-Mt. \]
Relative to $M=0$, the second shape derivative of $G_\gamma(t)$ consequently changes by $-Mt^2/2$. Substituting these changes into the loss yields
\[ \partial_{\gamma\gamma}\losspop_M(\vbeta')-\partial_{\gamma\gamma}\losspop_0(\vbeta')=\frac M2\sum_{x=0}^1\left[m_x-\frac{\lambda'_x}{2}\E_x^*[T^2]-\frac{\lambda'_x\tau^2}{2(e^{\lambda'_x\tau}-1)}\right]=-MA_1. \]
Therefore,
\[ \nabla^2\losspop_M(\vbeta')=\begin{pmatrix}D&w\\w&b-MA_1\end{pmatrix}. \]
All entries are finite, since the sampled densities and their parameter derivatives are smooth and locally bounded on the bounded observation interval. Because $A_1<0$, increasing $M$ increases the baseline-direction curvature without changing the other Hessian entries. Fix
\[ M:=1+\frac{\max\{0,w^2/D-b\}}{-A_1}. \]
Then $b-MA_1>w^2/D$, so the Schur complement gives $\nabla^2 \losspop_M(\vbeta')\succ0$. Writing $\losspop=\losspop_M$ for this fixed choice, we have
\[ \nabla \losspop(\vbeta')=\mathbf{0} \text{ and } \nabla^2\losspop(\vbeta')\succ 0. \]
So $\vbeta'$ is a strict local minimum in the interior of the parameter space.

We next show that no two distinct parameter pairs generate the same sampled conditional distributions. This establishes identifiability and rules out an alternative representation of the true model. Suppose $\vbeta=(\theta,\gamma)$ and $\widetilde\vbeta=(\widetilde\theta,\widetilde\gamma)$ generate the same sampled densities for both covariate groups. Since
\[ \frac{d}{dt}\log\frac{p_1(t;\vbeta)}{p_0(t;\vbeta)}=-(e^\theta-1)g_\gamma(t), \]
equality of the densities implies
\[ (e^\theta-1)g_\gamma(t)=(e^{\widetilde\theta}-1)g_{\widetilde\gamma}(t). \]
Taking $t\downarrow0$ and using $g_\gamma(0)=1$ gives $\theta=\widetilde\theta$. Since $\theta>0$, the common multiplier is nonzero, so $g_\gamma=g_{\widetilde\gamma}$ on the retained interval. Moreover,
\[ \frac{d^3}{dt^3}\log g_\gamma(t)=\frac{2\gamma^3}{(1-t)^3}. \]
The polynomial terms used to control the gradient and curvature disappear after three derivatives. Equality of the baseline shapes therefore implies $\gamma^3=\widetilde\gamma^3$, hence $\gamma=\widetilde\gamma$. We conclude that the joint model is identifiable.

Finally, the population loss satisfies
\[ \losspop(\vbeta)-\losspop(\vbeta^*)=\frac12\sum_{x=0}^1 D_{\mathrm{KL}} \left(p_x^* \middle\|p_x(\cdot;\vbeta)\right) \ge0, \]
with equality only at $\vbeta=\vbeta^*$ by identifiability. So $\vbeta^*$ is the unique global minimizer. Since $\theta'<\theta^*$, the strict local minimum $\vbeta'$ satisfies $\losspop(\vbeta')>\losspop(\vbeta^*)$, completing the proof.
\end{proof}

\CoxFlatProfile*

\begin{proof}
Consider a binary covariate $x\in\{0,1\}$, with both values occurring with probability $1/2$ in the sampled population. We retain individuals exactly when $0<T^*<1$ and assume no censoring. Let $\theta^*=\log 2$ and $h_0^*(t)=1$. The true sampled densities are
\[ p_0^*(t)=\frac{e^{-t}}{1-e^{-1}} \text{ and } p_1^*(t)=\frac{2e^{-2t}}{1-e^{-2}} \text{ for } 0<t<1. \]

For a candidate parameter $(\theta,h_0)$ with $h_0\in\mathcal H$, let $H_0(t)=\int_0^t h_0(s)ds$. The corresponding sampled density for covariate $x$ is
\[
p_x(t;\theta,h_0)=\frac{e^{x\theta}h_0(t)e^{-e^{x\theta}H_0(t)}}{1-e^{-e^{x\theta}H_0(1)}} \text{ for } 0<t<1.
\]
The population negative log-likelihood is
\[
\losspop(\theta,h_0)=-\frac12\sum_{x=0}^1\int_0^1 p_x^*(t)\log p_x(t;\theta,h_0)dt.
\]
Define the pooled true density and its loss by
\[
p_{\mathrm{pool}}(t):=\frac{p_0^*(t)+p_1^*(t)}{2} \text{ and } \losspop_{\mathrm{pool}}:=-\int_0^1 p_{\mathrm{pool}}(t)\log p_{\mathrm{pool}}(t)dt.
\]

Now, fix $\theta\le0 < \theta^*$. For every $h_0\in\mathcal H$,
\[
\frac{d}{dt}\log\frac{p_1(t;\theta,h_0)}{p_0(t;\theta,h_0)}=-(e^\theta-1)h_0(t)\ge0,
\]
so the candidate likelihood ratio $p_1/p_0$ is nondecreasing. In contrast,
\[
\frac{p_1^*(t)}{p_0^*(t)}=\frac{2(1-e^{-1})}{1-e^{-2}}e^{-t}
\]
is strictly decreasing. Therefore, when $\theta\le0$, the candidate model does not have the same strictly decreasing ratio of the likelihoods that the true model has.

We first show that every such candidate has population loss at least $\losspop_{\mathrm{pool}}$. Define
\[
\eta(t):=\frac{p_1^*(t)}{2p_{\mathrm{pool}}(t)}.
\]
Since $p_1^*/p_0^*$ is strictly decreasing, $\eta$ is strictly decreasing. In addition,
\[
p_0^*(t)=2p_{\mathrm{pool}}(t)(1-\eta(t)) \text{ and } p_1^*(t)=2p_{\mathrm{pool}}(t)\eta(t).
\]
Using $-\log v\ge1-v$, we obtain
\[
\losspop(\theta,h_0)-\losspop_{\mathrm{pool}}\ge\frac12\sum_{x=0}^1\int_0^1 p_x^*(t)\left(1-\frac{p_x(t;\theta,h_0)}{p_{\mathrm{pool}}(t)}\right)dt.
\]
Substituting the expressions above for $p_0^*$ and $p_1^*$ gives
\[
\losspop(\theta,h_0)-\losspop_{\mathrm{pool}}\ge-\int_0^1\eta(t)\left(p_1(t;\theta,h_0)-p_0(t;\theta,h_0)\right)dt.
\]
Since both candidate densities integrate to one, the right-hand side can be written as
\[
-\operatorname{Cov}_{T\sim p_0(\cdot;\theta,h_0)}\left(\eta(T),\frac{p_1(T;\theta,h_0)}{p_0(T;\theta,h_0)}\right).
\]
The covariance is nonpositive because $\eta$ is decreasing while $p_1/p_0$ is nondecreasing, so these two functions are oppositely monotone under $T\sim p_0(\cdot;\theta,h_0)$. Hence
\[
\losspop(\theta,h_0)\ge\losspop_{\mathrm{pool}}.
\]
When $\theta<0$, the ratio $p_1/p_0$ is strictly increasing because $h_0(t)>0$, while $\eta$ is strictly decreasing. Since $p_0(\cdot;\theta,h_0)$ has positive density on $(0,1)$, the covariance is strictly negative and therefore
\[
\losspop(\theta,h_0)>\losspop_{\mathrm{pool}}
\]
for every $h_0\in\mathcal H$.

We next show that $\losspop_{\mathrm{pool}}$ can nevertheless be approached arbitrarily closely for every fixed $\theta\le0$. Extend the formula for $p_{\mathrm{pool}}$ smoothly to $t\ge0$. Let $\chi:[0,\infty)\to[0,1]$ be smooth, equal to one on $[0,1]$, and equal to zero on $[2,\infty)$. For $\varepsilon>0$, define
\[
h_{0,\varepsilon}(t):=\varepsilon\chi(t)p_{\mathrm{pool}}(t)+1-\chi(t).
\]
This baseline belongs to $\mathcal H$ since it is smooth and strictly positive, and it equals one for $t\ge2$, so its cumulative hazard diverges at infinity. On the retained interval, $h_{0,\varepsilon}(t)=\varepsilon p_{\mathrm{pool}}(t)$. Since $\int_0^1 p_{\mathrm{pool}}(t)dt=1$, the corresponding sampled density satisfies
\[
\frac{p_{x,\varepsilon}(t)}{p_{\mathrm{pool}}(t)}=\frac{\varepsilon e^{x\theta}\exp\left[-\varepsilon e^{x\theta}\int_0^t p_{\mathrm{pool}}(s)ds\right]}{1-e^{-\varepsilon e^{x\theta}}}.
\]
As $\varepsilon\downarrow0$, the right-hand side converges uniformly to one on $[0,1]$. Therefore,
\[
\losspop(\theta,h_{0,\varepsilon})\longrightarrow\losspop_{\mathrm{pool}}.
\]
Combining this with the lower bound above yields
\[
\widehat{\losspop}(\theta)=\losspop_{\mathrm{pool}}
\]
for every $\theta\le0$. So, the profiled loss is constant on $(-\infty,0]$. Moreover, when $\theta<0$, the strict lower bound shows that no baseline $h_0\in\mathcal H$ attains this infimum.

It remains to verify that the true regression coefficient is identifiable. Suppose some $(\theta,h_0)$ satisfies $p_x(\cdot;\theta,h_0)=p_x^*$ for both $x=0$ and $x=1$. Equality of the derivatives of the two log-density ratios gives
\[
(e^\theta-1)h_0(t)=1 \text{ for } 0<t<1.
\]
So $h_0$ is constant on the retained interval. For $x=0$, the logarithmic derivative of the candidate density is $-h_0$, while the logarithmic derivative of $p_0^*$ is $-1$. We conclude that $h_0(t)=1$ on $(0,1)$, and the preceding identity gives $e^\theta-1=1$. Hence $\theta=\log2=\theta^*$, so the true regression coefficient is identifiable.

Finally, because the true model belongs to the candidate class, for every
$h_0\in\mathcal H$,
\[
\losspop(\theta^*,h_0)-\losspop(\theta^*,h_0^*)=\frac12\sum_{x=0}^1D_{\mathrm{KL}}\left(p_x^*\middle\|p_x(\cdot;\theta^*,h_0)\right)\ge0.
\]
Hence
\[
\widehat{\losspop}(\theta^*)=\losspop(\theta^*,h_0^*).
\]
Moreover,
\[
\losspop_{\mathrm{pool}}-\widehat{\losspop}(\theta^*)=\frac12\sum_{x=0}^1D_{\mathrm{KL}}\left(p_x^*\middle\|p_{\mathrm{pool}}\right)>0,
\]
because $p_0^*\ne p_1^*$. Therefore, $\widehat{\losspop}$ is constant and strictly suboptimal on $(-\infty,0]$. In particular, every $\theta<0$ is a suboptimal stationary point of the profiled population loss.
\end{proof}

\section{Further details on the experiments}
\label{sec:appendix-experimental-details}

We use four datasets for distinct clinical applications: HEARTS, SUPPORT, METABRIC, and FLCHAIN. HEARTS contains patient-donor matches for predicting post-transplant graft survival from the UNOS patient registry. SUPPORT contains data from hospitalized adults with the goal of predicting time to death. The dataset contains 8,873 patients and 14 clinical covariates. METABRIC contains 1,903 patients with breast cancer and 9 clinical and molecular covariates. FLCHAIN contains 7,871

individuals from a population-based study of serum free light chains and mortality. We use 8 of the 9 baseline covariates, excluding the cause-of-death variable. All experiments are conducted on an M4 Pro processor with 24GB unified memory. Other than the UNOS registry, the other datasets are publicly available in Python through standard libraries. 

For each dataset, we randomly split patients into training, validation, and test sets containing $70\%$, $15\%$, and $15\%$ of the cohort, respectively. All data preprocessing is fixed using the training split only. Continuous covariates are median-imputed when necessary and standardized to zero mean and unit variance. Missing categorical values are represented by a separate category, and categorical covariates are one-hot encoded. The preprocessing steps are applied without modification to the validation and test sets. We fix a random seed of $20260914$.

\begin{table}[h!]
\centering
\small
\resizebox{0.8\linewidth}{!}{
\begin{tabular}{lrrrr}
\toprule
Dataset & $N$ & Covariates &  Censoring rate & Median event time \\
\midrule
HEARTS     & $59,570$ & $95$ & $49.8\%$ & $5.7$ years \\
SUPPORT  & $8,873$  & $14$ & $32.0\%$ & $231$ days \\
METABRIC & $1,903$  & $9$  & $42.0\%$ & $115$ months \\
FLCHAIN  & $7,871$  & $8$ & $72.5\%$ & $4,303$ days \\
\bottomrule
\end{tabular}}
\caption{Details of datasets used in the semi-synthetic experiments.}
\label{tab:datasets}
\end{table}

\subsection{Semi-synthetic data generation}

The original observed survival times and censoring indicators from the datasets are used to fit the survival model that we ultimately use to construct the semi-synthetic data. 

We fit both a standard Cox regression model trained using the standard Cox partial likelihood and the original event and censoring indicators and a neural network model (see~\Cref{sec:appendix-neural}). We select the model using the validation split and then freeze its parameters. The fitted model is used only to define the semi-synthetic data-generating process and is not included among the methods evaluated in our experiments.

Before fitting the baseline hazard, we adjust the log-risk scores by their mean to address the arbitrary offset in the Cox risk score.  Define
\[
c=\frac{1}{|\mathrm{train}|}\sum_{i\in\mathrm{train}} r(\vx_i),
\]
where $r$ is the regression or neural network component of the Cox model. Then compute the ground-truth risk function as
\[
r^*(\vx)=r(\vx)-c.
\]
Subtracting the mean fixes the otherwise arbitrary additive offset in the Cox risk score. We absorb the corresponding shift into the baseline hazard so that the resulting conditional survival distribution is unchanged.

We next estimate the cumulative baseline hazard from the original training data using the Breslow estimator and $r^*(\vx)$, denoted $H_0^*$. Since the Breslow estimator is a step function, we linearly interpolate between the discrete points to obtain a continuous, nondecreasing cumulative hazard. We extrapolate using a positive slope at the final point. This results in a continuous function $H_0^*$ with a computable inverse and with $H_0^*(t)\to\infty$ as $t\to\infty$. Both $r^*$ and $H_0^*$ are fixed and treated as the exact ground-truth model in all subsequent synthetic experiments.

For each patient $\vx_i$, we independently sample $E_i\sim\mathrm{Exponential}(1)$ and generate the latent event time
\[
T_i^* = (H_0^*)^{-1} \left(E_i e^{-r^*(\vx_i)}\right).
\]
The generated outcomes follow the proportional-hazards model defined by
\[
h^*(t\mid\vx) = h_0^*(t)\exp(r^*(\vx)).
\]
In these synthetic experiments, the true event time and conditional survival distribution are known for every patient. Truncation and right censoring are then imposed synthetically based on the generated event times. All models are evaluated on the full, untruncated synthetic test set. This allows us to measure how well each method recovers the underlying risk function and survival distribution rather than relying only on performance against censored observational outcomes.

\subsection{Synthetic truncation and censoring}

We impose synthetic right censoring and truncation based on the generated event time. The censoring mechanism is fixed separately for each dataset, while we parameterize the severity of truncation which we vary in the experiments.

For each dataset, let $r$ denote the fraction of data points that are right-censored in the original data. We independently generate censoring times $C_i$ for each patient $i$ from an exponential distribution. We select the rate of this exponential so that the semi-synthetic training population has the same censoring rate $r$. We then observe
\[ T_i=\min\{T_i^*,C_i\} \text{ and } \delta_i=\mathbb I\{T_i^*\le C_i\}. \]
The censoring distribution is calibrated once for each dataset, and then fixed across the rest of the experiments. 

We implement a multi-interval weighted truncation mechanism with tunable strength. To obtain comparable truncation mechanisms across datasets with different time scales, we define the intervals using quartiles of the event-time distribution in the synthetic training data. Let $c_1,c_2,c_3$ denote the $25$th, $50$th, and $75$th percentiles, and define
\[
I_1=(0,c_1),
I_2=(c_1,c_2),
I_3=(c_2,c_3), \text{ and }
I_4=(c_3,\infty).
\]
Conditional on $T_i^*\in I_k$, patient $i$ is retained independently with probability
$q_k(\zeta)=e^{-\zeta a_k}$,  where we define $a$ for each interval as
$(a_1,a_2,a_3,a_4)=\left(\frac12,\frac32,1,2\right)$.

The parameter $\zeta$ controls the intensity of truncation. When $\zeta=0$, all observations are retained. Increasing $\zeta$ decreases the probability of observing patients in every interval, but at different rates. The coefficients are deliberately nonmonotone in event time. In particular, the second quartile is downsampled more aggressively than the third. For each patient $i$, we retain them independently with probability $q_k(\zeta)$ when $T_i^*\in I_k$.

We run the experiments with $\zeta\in\{0,0.25,0.5,1,1.5,2\}$, corresponding to expected removed percents $\{0, 26, 44, 67, 79, 86\}\%$. We perform 10 independent runs per truncation level. 

\subsection{Projected SGD settings}

We implement~\Cref{alg:minibatch-sgd} with the following hyperparameter settings: learning rate $10^{-3}$, weight decay $10^{-4}$, batch size $512$, at most $500$ epochs with early stopping patience $30$. When the data-generating model is Cox regression, we use the projection set that is the Euclidean ball with radius $5$. When the data-generating model is a neural network, we use the same projection set.

\subsection{Neural network model}
\label{sec:appendix-neural}

For each dataset we augment the Cox model with a neural risk function over the covariates. This model is commonly known as \textit{DeepSurv}~\citep{Katzman18:Deepsurv}. The hazard is  $h(t\mid x)=h_0(t)\exp(r_{\phi}(x))$
where $r_\phi:\mathbb R^d\rightarrow\mathbb R$ is a multilayer perceptron that outputs a scalar. We use two hidden layers of width $64$, with ReLU activations and dropout probability $0.1$, followed by a linear scalar output. 

All neural network models are trained with Adam using learning rate $10^{-3}$, weight decay $10^{-4}$, and batch size $512$, for at most $500$ epochs with early stopping patience $30$. They therefore differ only in whether the loss accounts for the truncation mechanism. Both the standard and truncated model are initialized with identical weights and differ only in the loss. 

The standard neural network ignores the truncation mechanism and minimizes the ordinary negative log-likelihood. For patient $i$, let $z_i=r_\phi(\vx_i)$ and  $S_i(t;\phi)=\exp\left(-e^{z_i}H_0^*(t)\right)$. Up to terms independent of $\phi$, the standard loss is $f_i(\phi) = -\delta_i z_i + e^{z_i}H_0^*(T_i)$.

The truncated neural network instead conditions this likelihood on inclusion. For a single $z_i$, recall the probability of inclusion is $\pi_i(\phi) = \sum_{k=1}^K q_k\left[ S_i(t_{0,k};\phi)-S_i(t_{1,k};\phi) \right]$. The truncated negative log-likelihood loss is then
\[ f_i^{\mathrm{trunc}}(\phi) = \delta_i \left[-z_i+e^{z_i}H_0^*(T_i)\right] - (1-\delta_i)\log\pi_i(T_i;\phi) + \log\pi_i(\phi). \]
This truncated loss is exactly the likelihood of the original truncated Cox model with the linear predictor $\vx_i^\top\vtheta$ replaced by the neural network $r_\phi(\vx_i)$.

\subsection{Enlarging the HEARTS dataset}
\label{appendix:enlarged-hearts}

We augment the HEARTS dataset to construct a synthetic dataset with arbitrarily many samples and covariates. We start by fixing a data-generating linear Cox model as before (after dropping collinear features and applying whitening). Given a target population size $N$ and dimension $d$, we set the number of synthetic features to be the difference between the actual HEARTS features and the target number $d$. We then draw indices $s_1,\dots,s_N$ uniformly with replacement from the real training samples, so each $x_{s_i}$ is a sample from the original HEARTS dataset. We append Gaussian features by drawing $z_j\sim\mathcal{N}(0,1)$ independently for each added covariate. We reuse the regression coefficients for the original features. For the synthetic features we draw a single fixed coefficient $\vtheta_j\sim\mathcal N(0,1)$. We scale the sampled coefficients to unit norm, then rescale it so that the synthetic features contribute half the variance of the risk score from the original features.  We apply censoring and truncation exactly the same as before. 

\section{Known-Baseline Expectation Maximization}
\label{appendix:em}

We compare our first-order optimization method against a known-baseline \textit{expectation maximization (EM)} adaptation of the missing-data approach of~\citet{Rennert22:Cox}. To account for truncation, we use the missing-data approach in which each data point represents an expected number of individuals inversely proportional to its probability of inclusion. At each iteration, the E-step computes this expected number of represented individuals and their expected total transformed event time under the current parameter estimate. The M-step then maximizes the corresponding expected complete-data Cox log-likelihood.

For observation $i$, let $\lambda_i=e^{\vx_i^\top\vtheta}$ and let $\pi_i(\vtheta)$ denote its probability of inclusion. The E-step quantities are
\[
N_i=\frac{1}{\pi_i(\vtheta)}
\]
and
\[
A_i=\delta_iH_0(T_i)+(1-\delta_i)\mu_i(T_i\mid\vtheta)+\frac{1}{\lambda_i\pi_i(\vtheta)}-\mu_i(\vtheta).
\]
Holding these quantities fixed, the M-step maximizes
\[
Q(\vtheta)=\sum_{i=1}^n\left[N_i\vx_i^\top\vtheta-A_i e^{\vx_i^\top\vtheta}\right].
\]
This objective is concave and is optimized using Newton's method. The quantities $\pi_i$, $\mu_i$, and $\mu_i(T_i\mid\vtheta)$ are computed in closed form using the same interval-weighted exponential formulas as our likelihood implementation.

We use the same stopping condition as the first-order methods and the same regularization as in projected SGD. The full procedure is given in~\Cref{alg:known-baseline-em}.

\begin{algorithm}[h]
\caption{Known-baseline EM for truncated Cox regression}
\label{alg:known-baseline-em}
\begin{algorithmic}[1]
\Require Retained data $\{(\vx_i,T_i,\delta_i)\}_{i=1}^n$, known $H_0$
\State Initialize $\vtheta^{(0)}=\mathbf{0}$

\For{$m=0,1,\ldots$}
    \For{$i=1,\ldots,n$}
        \State $\lambda_i\gets e^{\vx_i^\top\vtheta^{(m)}}$, $\pi_i\gets\pi_i(\vtheta^{(m)})$, $N_i\gets1/\pi_i$
        \State $A_i\gets\delta_iH_0(T_i)+(1-\delta_i)\mu_i(T_i\mid\vtheta^{(m)})+\frac{1}{\lambda_i\pi_i}-\mu_i(\vtheta^{(m)})$
    \EndFor
    \State $\vtheta^{(m+1)}\gets\arg\max_{\vtheta}\sum_i\left[N_i\vx_i^\top\vtheta-A_i e^{\vx_i^\top\vtheta}\right]$
    \If{converged}
        \State \Return $\vtheta^{(m+1)}$
    \EndIf
\EndFor
\end{algorithmic}
\end{algorithm}

\section{Omitted tables and figures}

\begin{figure*}[h!]
    \centering
    \begin{subfigure}[b]{0.24\textwidth}
        \centering
        \includegraphics[width=\linewidth]{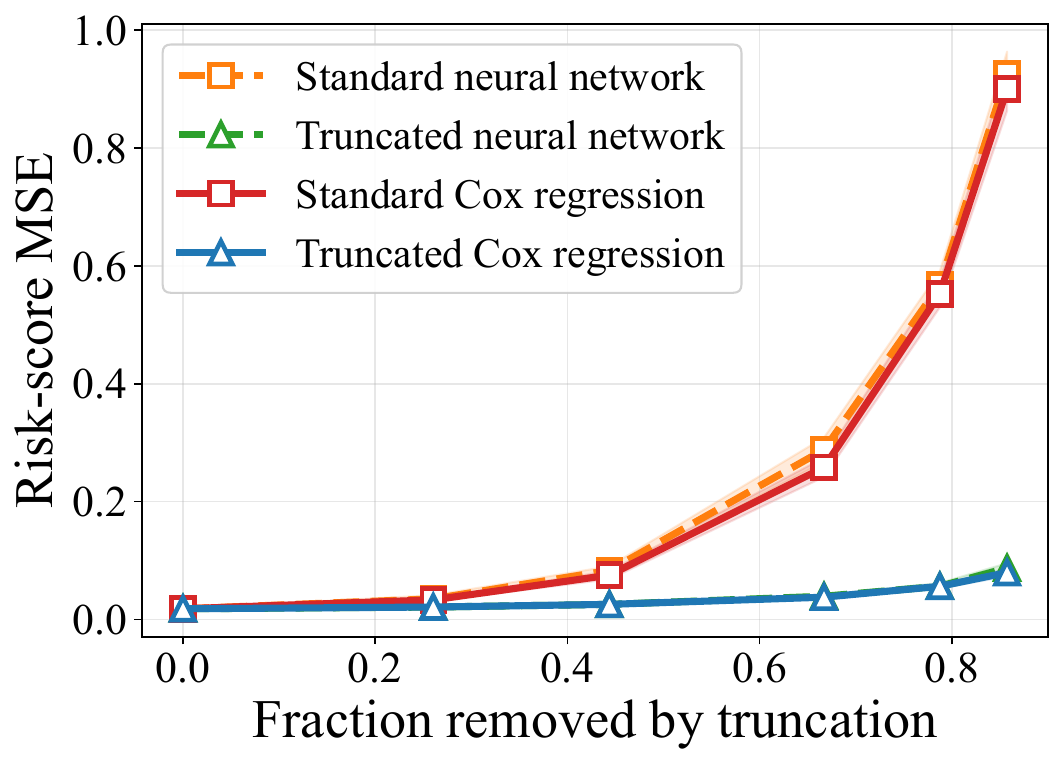}
        \caption{HEARTS.}
        \label{fig:unos-neural}
    \end{subfigure}
    \hfill
    \begin{subfigure}[b]{0.24\textwidth}
        \centering
        \includegraphics[width=\linewidth]{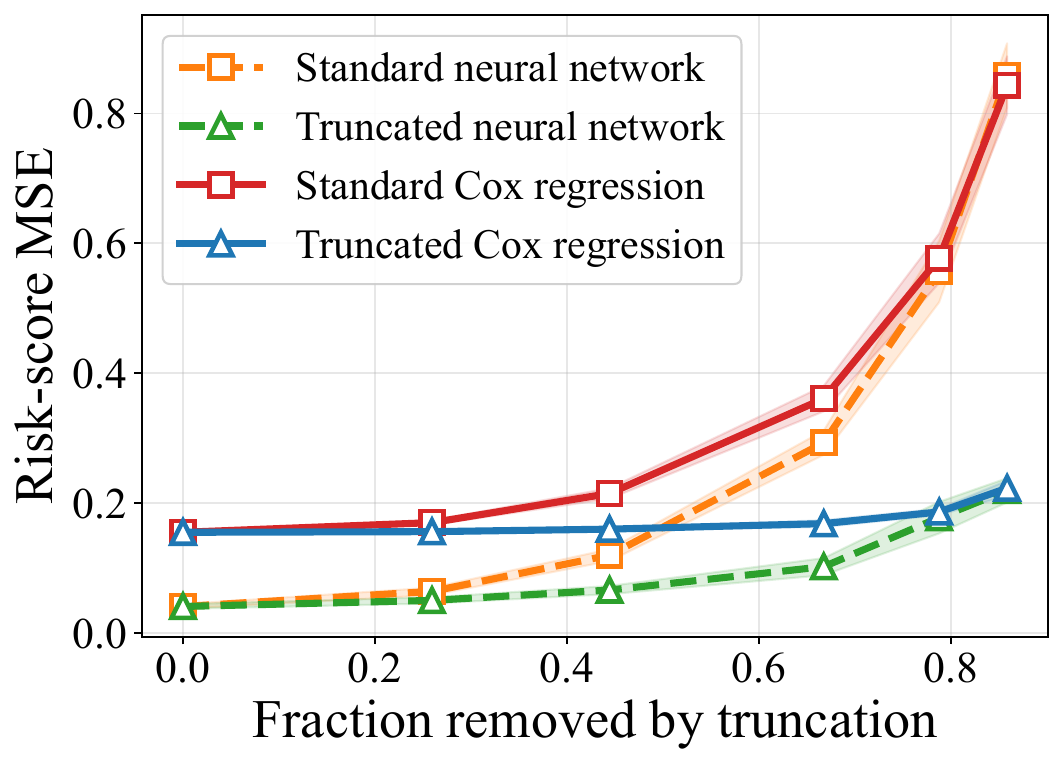}
        \caption{SUPPORT.}
        \label{fig:support-neural}
    \end{subfigure}
    \hfill
    \begin{subfigure}[b]{0.24\textwidth}
        \centering
        \includegraphics[width=\linewidth]{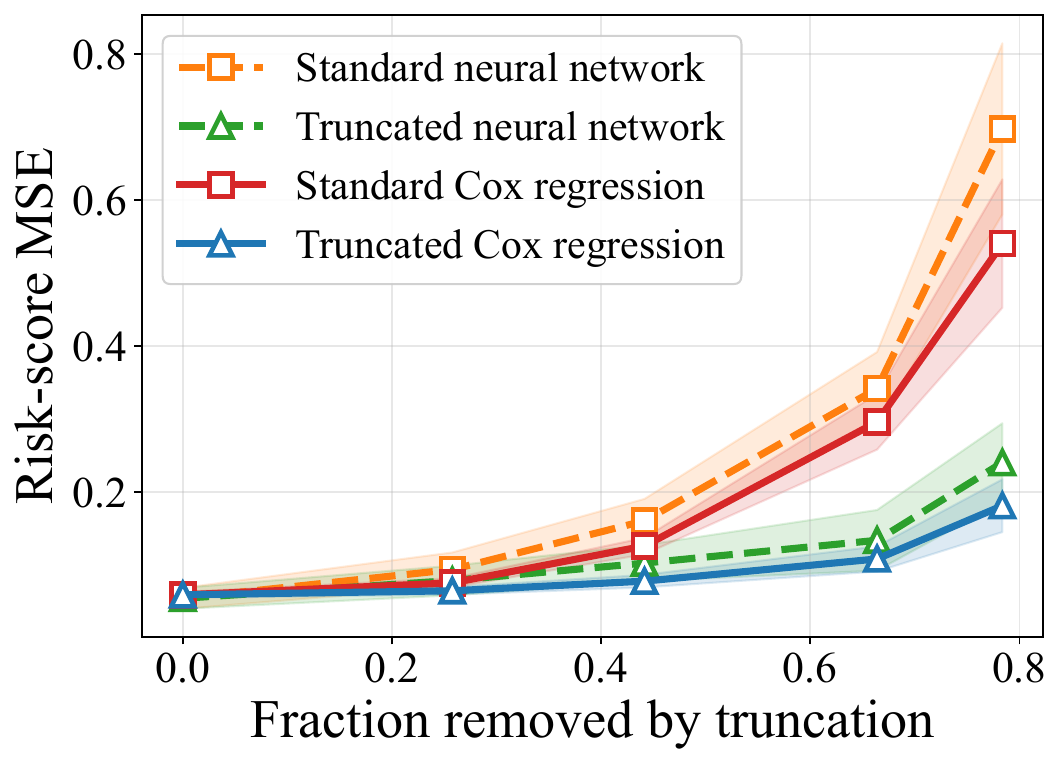}
        \caption{METABRIC.}
        \label{fig:metabric-neural}
    \end{subfigure}
    \hfill
    \begin{subfigure}[b]{0.24\textwidth}
        \centering
        \includegraphics[width=\linewidth]{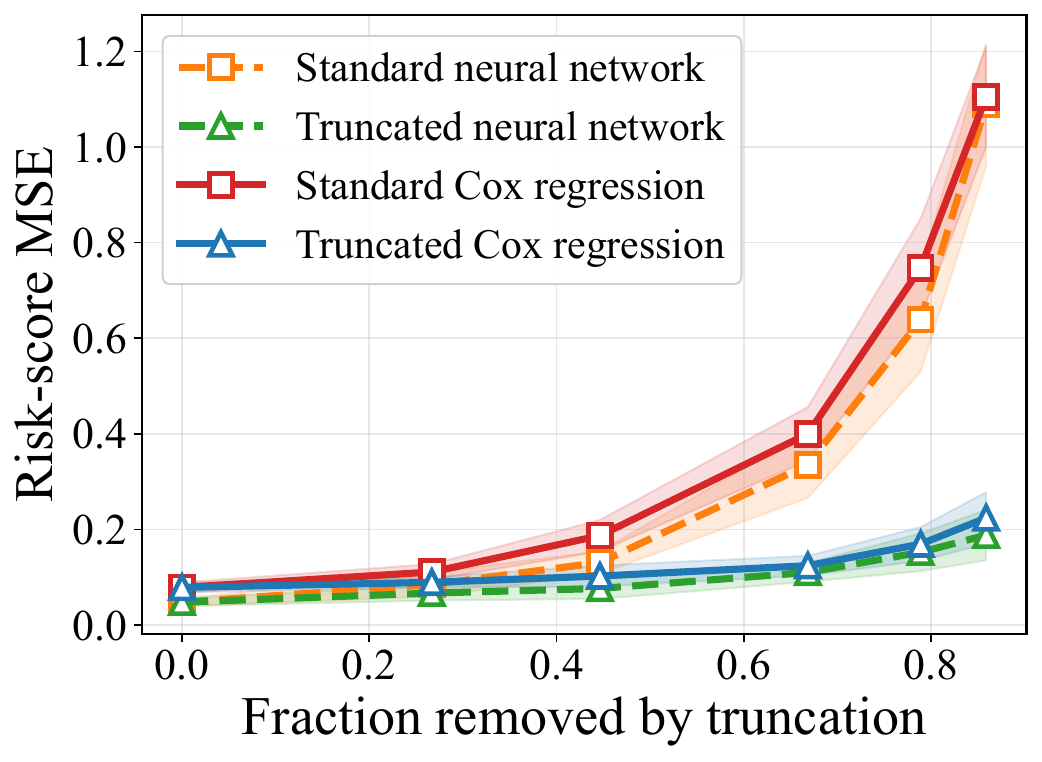}
        \caption{FLCHAIN.}
        \label{fig:flchain-neural}
    \end{subfigure}
    
    \caption{Results of learning standard and truncated Cox regression and neural networks with first-order optimizers across datasets when the data-generating model is a neural network.}
    \label{fig:dataset-mse-neural}
\end{figure*}

\begin{table}[h!]
  \centering
  \resizebox{0.5\linewidth}{!}{
  \begin{tabular}{l rr rr rr}
    \toprule
      & \multicolumn{2}{c}{Adam \textcolor{gray}{(Ours)}}
      & \multicolumn{2}{c}{Projected SGD \textcolor{gray}{(Ours)}}
      & \multicolumn{2}{c}{EM} \\
    \cmidrule(lr){2-3}\cmidrule(lr){4-5}\cmidrule(lr){6-7}
    $\zeta$ & Time (s) & MSE & Time (s) & MSE & Time (s) & MSE \\
    \midrule
     $0.0$  & $61.8$ & $0.0058$ & $38.4$ & $0.0058$ & $4855.9$ & $0.0054$ \\
    $0.25$ & $34.1$ & $0.0078$ & $33.7$ & $0.0078$ & $3127.0$ & $0.0077$ \\
    $0.5$  & $22.8$ & $0.0109$ & $29.0$ & $0.0108$ & $2065.7$ & $0.0107$ \\
    $1.0$  & $16.9$ & $0.0228$ & $14.8$ & $0.0226$ & $1803.2$ & $0.0223$ \\
    $1.5$  & $11.2$ & $0.0426$ & $8.2$ & $0.0411$ & $1740.0$ & $0.0415$ \\
    $2.0$  & $5.6$ & $0.0788$ & $7.7$ & $0.0772$ & $2050.9$ & $0.0755$ \\
    \bottomrule
  \end{tabular}}
  \caption{Optimizer comparison on HEARTS-Large.}
    \label{tab:scalability}
\end{table}

\end{document}